\documentclass{article}

\usepackage{microtype}
\usepackage{graphicx}
\usepackage{subfigure}
\usepackage{booktabs} 

\usepackage{hyperref}

\usepackage[accepted]{icml2020}

\icmltitlerunning{Combinatorial Pure Exploration of Dueling Bandit  }

\usepackage{amsthm}

\newtheorem{lemma}{Lemma}

\newtheorem{definition}{Definition}

\usepackage{amsmath,amssymb}
\usepackage{amsfonts}
\usepackage{enumerate}
\usepackage{flushend}
\usepackage{mathrsfs}
\usepackage{subfigure}
\usepackage{booktabs}
\usepackage{makecell}
\usepackage{setspace}

\newcommand{\R}{\mathbb{R}}

\newcommand{\cB}{\mathcal{B}}

\newcommand{\cG}{\mathcal{G}}

\newcommand{\cM}{\mathcal{M}}
\newcommand{\cN}{\mathcal{N}}

\newcommand{\cP}{\mathcal{P}}
\newcommand{\cQ}{\mathcal{Q}}

\newcommand{\cX}{\mathcal{X}}

\usepackage{ifthen}

\usepackage{thm-restate}

\usepackage{tikz}
\usetikzlibrary{positioning,chains,fit,shapes,calc}

\newcommand{\argmax}{\operatornamewithlimits{argmax}}
\newcommand{\argmin}{\operatornamewithlimits{argmin}}

\newcommand{\compilefullversion}{true}
\ifthenelse{\equal{\compilefullversion}{false}}{%
    \newcommand{\OnlyInFull}[1]{}
    \newcommand{\OnlyInShort}[1]{#1}
}{%
    \newcommand{\OnlyInFull}[1]{#1}%
    \newcommand{\OnlyInShort}[1]{}%
}%

\newcommand{\compilehidecomments}{true}
\ifthenelse{ \equal{\compilehidecomments}{true} }{%
	\newcommand{\wei}[1]{}
	\newcommand{\haoyu}[1]{}
	\newcommand{\yihan}[1]{}
	\newcommand{\longbo}[1]{}
}{
	\newcommand{\wei}[1]{{\color{blue!50!black}  [\text{Wei:} #1]}}
	\newcommand{\haoyu}[1]{{\color{brown!60!black} [\text{Haoyu:} #1]}}
	\newcommand{\yihan}[1]{{\color{brown!60!black} [\text{Yihan:} #1]}}
		\newcommand{\longbo}[1]{{\color{red!60!black} [\text{Longbo:} #1]}}
}

\mathchardef\mhyphen="2D
\newcommand{\Oracle}{{\sf O}}
\newcommand{\MaxMatchingOracle}{{\sf MWMC}}
\newcommand{\Carcond}{{\sf CAR\mhyphen Cond}}

\newcommand{\Carverify}{{\sf CAR\mhyphen Verify}}
\newcommand{\Carparallel}{{\sf CAR\mhyphen Parallel}}
\newcommand{\Cborda}{{\sf CLUCB\mhyphen Borda\mhyphen PAC}}
\newcommand{\CbordaExact}{{\sf CLUCB\mhyphen Borda \mhyphen Exact}}
\newcommand{\Clucb}{{\sf CLUCB}}

\newcommand{\Out}{{\sf{Out}}}
\newcommand{\Width}{\textup{width}}

\usepackage{cleveref}
\crefname{equation}{Eq.}{Eqs.}
\crefname{section}{Section}{Sections}
\crefname{algorithm}{Algorithm}{Algorithms}
\crefname{figure}{Figure}{Figures}
\crefname{table}{Table}{Tables}
\crefname{theorem}{Theorem}{Theorems}
\crefname{lemma}{Lemma}{Lemmas}
\crefname{corollary}{Corollary}{Corollaries}
\crefname{definition}{Definition}{Definitions}

\allowdisplaybreaks

\begin{document}

\twocolumn[
\icmltitle{Combinatorial Pure Exploration for Dueling Bandits}



\icmlsetsymbol{equal}{*}

\begin{icmlauthorlist}
\icmlauthor{Wei Chen}{equal,msra}
\icmlauthor{Yihan Du}{iiis}
\icmlauthor{Longbo Huang}{iiis}
\icmlauthor{Haoyu Zhao}{iiis}
\end{icmlauthorlist}

\icmlaffiliation{msra}{Microsoft Research, Beijing, China}
\icmlaffiliation{iiis}{IIIS, Tsinghua University, Beijing, China}

\icmlcorrespondingauthor{Wei Chen}{weic@microsoft.com}
\icmlcorrespondingauthor{Yihan Du}{duyh18@mails.tsinghua.edu.cn}
\icmlcorrespondingauthor{Longbo Huang}{longbohuang@mail.tsinghua.edu.cn}
\icmlcorrespondingauthor{Haoyu Zhao}{thomaszhao1998@gmail.com}

\icmlkeywords{Machine Learning, ICML}

\vskip 0.3in
]



\printAffiliationsAndNotice{Alphabetical order}  

\begin{abstract}
    In this paper, we study combinatorial pure exploration for dueling bandits (CPE-DB): we have multiple candidates for multiple positions as modeled by a bipartite graph, and in each round we sample a duel of two candidates on one position and observe who wins in the duel, with the goal of finding the best candidate-position matching with high probability after multiple rounds of samples. CPE-DB is an adaptation of the original combinatorial pure exploration for multi-armed bandit (CPE-MAB) problem to the dueling bandit setting. 
    We consider both the Borda winner and the Condorcet winner cases. For Borda winner, we establish a reduction of the problem to the original CPE-MAB setting and design PAC and exact algorithms that achieve both the sample complexity similar to that in the CPE-MAB setting (which is nearly optimal for a subclass of problems)  and polynomial running time per round. 
    For Condorcet winner, we first design a fully polynomial time approximation scheme (FPTAS) for the offline problem of finding the Condorcet winner with known winning probabilities, and then use the FPTAS as an oracle to design a novel pure exploration algorithm $\Carcond$ with sample complexity analysis. $\Carcond$ is the first algorithm with polynomial running time per round for identifying the Condorcet winner in CPE-DB.
\end{abstract}

\section{Introduction}

Multi-Armed Bandit (MAB) \cite{lai1985asymptotically,thompson1933,UCB_auer2002,agrawal2012analysis} 
is a classic model that  characterizes the exploration-exploitation tradeoff  in online learning. 
The pure exploration task \cite{Elim2006_explore1,chen2016open,BAI2019epsilon} is an important variant of the MAB problems, 
where the objective is to identify the best arm with high confidence, using as few samples as possible. 
A rich class of pure exploration problems have been extensively studied, e.g., 
best K-arm identification \cite{lucb} and  multi-bandit best arm identification \cite{SAR_bubeck2013multiple}. 
Recently, \citet{cpe_icml14} proposes a general combinatorial pure exploration for multi-armed bandit (CPE-MAB) framework, 
which encompasses previous pure exploration problems. 
In the CPE-MAB problem, a learner is given a set of arms and a collection of arm subsets with certain combinatorial structures. 
At each time step, the learner plays an arm and observes the random reward, with the objective of identifying the best combinatorial subset of arms. 
\citet{gabillon2016improved,nearly_ChenLJ2017} follow this setting and further improve the sample complexity.

However, in many real-world applications involving implicit (human) feedback including  social surveys \cite{social_survey_1985measurement},  market research \cite{market1994combining} and recommendation systems \cite{recommendation_systems}, the information observed by the learner is intrinsically relative. For example, in voting and elections, it is more natural for the electors to offer preference choices than numerical evaluations on candidates. For this scenario, the dueling bandit formulation  \cite{IF_2012,tournament2016dueling,survey2018dueling} provides a promising model for online decision making with relative feedback.

In this paper, we contribute a model adapting the original CPE-MAB problem to the dueling bandit setting.
Specifically, we formulate the \emph{combinatorial pure exploration for dueling bandit (CPE-DB)} problem as follows.
A CPE-DB instance consists of a bipartite graph $G$ modeling multiple candidates that could fit into multiple positions, and an {\em unknown preference probability matrix} specifying when we play a duel between two candidates for one position, the probability that the first would win over the second.
At each time step, a learner samples a duel of two candidates on one position and observes a random outcome of which candidate wins in this duel sampled according to the preference probability matrix.
The objective is to use as few duel samples as possible to identify the best candidate-position matching with high confidence, for two popular optimality metrics in the dueling bandit literature, i.e., Condorcet winner and Borda winner.


The CPE-DB model represents a novel preference-based version of the common candidate-position matching problems, which occurs in various real-world scenarios, including social choice \cite{social_choice}, multi-player game \cite{multi-player_game} and online advertising \cite{online_advertising}. For instance, a committee selection procedure \cite{committee_selection} may want to choose among multiple candidates one candidate for each position
    to form a committee.
For any two candidates on one position, 
    we can play a duel on them, e.g., by surveying a bystander, to learn a sample of which candidate would win on this position, and
    the sample follows an unknown preference probability.
We hope to play as few duels as possible (or by surveying as few people as possible) to identify the best performing committee.

The CPE-DB problem raises interesting challenges on exponentially large decision space and relative feedback. The key issue here is how to exploit the problem structure and design algorithms that guarantee both high computational efficiency and low sample complexity. Therefore, the design and analysis of algorithms for CPE-DB demand novel computational acceleration techniques. The contributions of this work are summarized as follows:
\begin{enumerate}[(1).]
    \vspace{-2mm}
    \item We formulate the combinatorial pure exploration for dueling bandit (CPE-DB) problem, adapted from the original combinatorial pure exploration for multi-armed bandit (CPE-MAB) problem to the dueling bandit setting, and associate it with
    various real-world applications involving preference-based bipartite matching selection.
    \vspace{-2mm}
    \item For the Borda winner metric, we reduce CPE-DB to the original CPE-MAB problem, and design algorithms $\Cborda$ and $\CbordaExact$ with polynomial running time per round. We provide their sample complexity upper bounds and a  problem-dependent lower bound for CPE-DB with Borda winner. 
    Our upper and lower bound results together show that $\CbordaExact$ achieves near-optimal sample complexity for a subclass of problems. 
    \vspace{-2mm}
    \item For the Condorcet winner metric, we design a fully polynomial time approximation scheme (FPTAS) for a proper extended version of the offline problem, and then adopt the FPTAS to design a novel online algorithm $\Carcond$ with sample complexity analysis. To our best knowledge, $\Carcond$ is the first algorithm with polynomial running time per round for identifying the Condorcet winner in CPE-DB.
\end{enumerate}

\subsection{Related Works}

\paragraph{Combinatorial pure exploration}
The combinatorial pure 
exploration for multi-armed bandit (CPE-MAB) problem is first formulated by \citet{cpe_icml14} and 
    generalizes the multi-armed bandit pure exploration task to general combinatorial structures. 
\citet{gabillon2016improved} follow the setting of \cite{cpe_icml14} and propose algorithms with improved sample complexity but a loss of computational efficiency.
\citet{nearly_ChenLJ2017} further design algorithms for this problem that have tighter sample complexity and pseudo-polynomial running time. 
\citet{wu2015identifying} study another combinatorial pure exploration case in which given a graph, at each time step, a learner samples a path with the objective of identifying the optimal edge.


\paragraph{Dueling bandit}
The dueling bandit problem \cite{IF_2012,tournament2016dueling,survey2018dueling}, 
    first proposed by \cite{IF_2012}, is an important variation of the multi-armed bandit setting. 
According to the assumptions on  preference structures and definitions of the optimal arm (winner), previous methods can be categorized as  methods on Condorcet winner \cite{komiyama2015regret,Qualitative2019dueling}, methods on Borda winner \cite{sparse2015borda,Qualitative2019dueling}, methods on Copeland winner \cite{wu2016doubleTS,copeland2019agrawal}, etc.  Recently, \citet{CombinatorialDueling2019} propose a variant of combinatorial bandits with relative feedback. 
In their setting, a learner plays a subset of arms (assuming each arm has an unknown positive value) in a time step and observes the ranking feedback, 
    and the goal is to minimize the cumulative regret.
Therefore, their model is quite different from ours.

\section{Problem Formulation}\label{sec:problem}

\begin{figure*}
\begin{minipage}[b]{.22\linewidth}
\centering

\definecolor{myblue}{RGB}{80,80,160}
\definecolor{mygreen}{RGB}{80,160,80}

\begin{tikzpicture}[thick,
  fsnode/.style={draw,circle,fill=myblue},
  ssnode/.style={draw,circle,fill=mygreen}
]

\begin{scope}[start chain=going below,node distance=7mm]
\foreach \i in {1,2,3,4}
  \node[fsnode,on chain] (c\i) [label=left: $c_{\i}$] {};
\end{scope}

\begin{scope}[xshift=2cm,yshift=-1cm,start chain=going below,node distance=7mm]
\foreach \i in {1,2}
  \node[ssnode,on chain] (s\i) [label=right: $s_{\i}$] {};
\end{scope}


\draw (c1) -- node[sloped,above=0.1mm] {$e_1$} (s1);
\draw (c2) -- node[sloped,above=0.1mm] {$e_2$} (s1);
\draw (c3) -- node[sloped,above=0.1mm] {$e_3$} (s1);
\draw (c3) -- node[sloped,above=0.1mm] {$e_4$} (s2);
\draw (c4) -- node[sloped,above=0.1mm] {$e_5$} (s2);
\end{tikzpicture}
\caption{Graph}
\label{fig:bipartite-graph}
\end{minipage}
\hfill
\begin{minipage}[b]{.3\linewidth}
\centering
\setlength{\tabcolsep}{1mm}{
\begin{tabular}[b]{cccccc}
\hline
\hline
      & $e_1$ & $e_2$ & $e_3$ & $e_4$ & $e_5$ \\
$e_1$ & 0.5 & 0.45 & 1 & 0 & 0\\
$e_2$ & 0.55 & 0.5 & 0.55 & 0 & 0\\
$e_3$ & 0 & 0.45 & 0.5 & 0 & 0\\
$e_4$ & 0 & 0 & 0 & 0.5 & 0 \\
$e_5$ & 0 & 0 & 0 & 1 & 0.5 \\
\hline
\hline
\end{tabular}
}
\caption{Preference Matrix}
\label{tab:probability}
\end{minipage}
\hfill
\begin{minipage}[b]{.22\linewidth}
\centering

\definecolor{myblue}{RGB}{80,80,160}
\definecolor{mygreen}{RGB}{80,160,80}

\begin{tikzpicture}[thick,
  fsnode/.style={draw,circle,fill=myblue},
  ssnode/.style={draw,circle,fill=mygreen}
]

\begin{scope}[start chain=going below,node distance=7mm]
\foreach \i in {1,2,3,4}
  \node[fsnode,on chain] (c\i) [label=left: $c_{\i}$] {};
\end{scope}

\begin{scope}[xshift=2cm,yshift=-1cm,start chain=going below,node distance=7mm]
\foreach \i in {1,2}
  \node[ssnode,on chain] (s\i) [label=right: $s_{\i}$] {};
\end{scope}


\draw[red] (c1) -- node[sloped,above=0.1mm] {$e_1$} (s1);
\draw[dashdotted] (c2) -- node[sloped,above=0.1mm] {$e_2$} (s1);
\draw[dashdotted] (c3) -- node[sloped,above=0.1mm] {$e_3$} (s1);
\draw[dashdotted] (c3) -- node[sloped,above=0.1mm] {$e_4$} (s2);
\draw[red] (c4) -- node[sloped,above=0.1mm] {$e_5$} (s2);
\end{tikzpicture}
\caption{Borda Winner}
\label{fig:borda-winner}
\end{minipage}
\hfill
\begin{minipage}[b]{.22\linewidth}
\centering

\definecolor{myblue}{RGB}{80,80,160}
\definecolor{mygreen}{RGB}{80,160,80}

\begin{tikzpicture}[thick,
  fsnode/.style={draw,circle,fill=myblue},
  ssnode/.style={draw,circle,fill=mygreen}
]

\begin{scope}[start chain=going below,node distance=7mm]
\foreach \i in {1,2,3,4}
  \node[fsnode,on chain] (c\i) [label=left: $c_{\i}$] {};
\end{scope}

\begin{scope}[xshift=2cm,yshift=-1cm,start chain=going below,node distance=7mm]
\foreach \i in {1,2}
  \node[ssnode,on chain] (s\i) [label=right: $s_{\i}$] {};
\end{scope}


\draw[dashdotted] (c1) -- node[sloped,above=0.1mm] {$e_1$} (s1);
\draw[red] (c2) -- node[sloped,above=0.1mm] {$e_2$} (s1);
\draw[dashdotted] (c3) -- node[sloped,above=0.1mm] {$e_3$} (s1);
\draw[dashdotted] (c3) -- node[sloped,above=0.1mm] {$e_4$} (s2);
\draw[red] (c4) -- node[sloped,above=0.1mm] {$e_5$} (s2);
\end{tikzpicture}
\caption{Condorcet Winner}
\label{fig:cond-winner}
\end{minipage}
\end{figure*}

In this section, we formally define the combinatorial pure exploration problem for dueling bandits. Suppose that there are $n$ candidates $C = \{c_1,\dots,c_n\}$ and $\ell$ positions $S = \{s_1,\dots,s_{\ell}\}$ with $n\ge \ell$. 
Each candidate is available for several positions, and we use  bipartite graph $G(C,S,E)$ to denote this relation, 
where each edge $e= (c_i,s_j)\in E$ denotes that candidate $c_i$ is capable for position $s_j$. We define $m = |E|$.
We use $E_j$ to denote the set of edges connected to position $j$, i.e., $E_j = \{e = (c,s_j)\in E:c\in C \}$ and we also use $s(e)$ to denote the position index of $e$.

Two edges $e$ and $e'$ are comparable if they have the same position indices, i.e. $s(e)=s(e')$.
For any two comparable edges $e=(c,s_j)$ and $e'=(c',s_j)$, there is an unknown  preference probability $p_{e,e'}$, which means that with probability $p_{e,e'}$, $e$ wins $e'$, or $c$ wins $c'$ on position $j$. We have $p_{e,e'} = 1-p_{e',e}$. 
For any $e\in E$, we define $p_{e,e} = \frac{1}{2}$ . 

Given the graph $G(C,S,E)$, we define an order of edges in $E$ by first ranking them by their position indices from smallest to the largest and 
    then ranking them by their candidate index from the smallest to the largest.
Given the order of the edges, we use $e_i$ to denote the $i$-th edge in the order, and define $\chi_M\in\{0,1\}^m$ as the vector representation of the edges $M\subset E$, where $(\chi_M)_i = 1$ if and only if $e_i\in M$.
We also use a  preference matrix $P\in ([0,1])^{m\times m}$ to record all preference probabilities. 
Specifically, for any two comparable edges $e_i, e_j$, $P_{i,j} = p_{e_i,e_j}$ is the preference probability of $e_i$ over $e_j$. 
For two incomparable edges $e_{i'},e_{j'}$, $P_{i',j'}$ is set to $0$ for the convenience of later computations.
Figure~\ref{fig:bipartite-graph} show an example bipartite graph and Figure~\ref{tab:probability} shows its corresponding preference matrix.

Note that for each position $s_j$, any two edges connecting to $s_j$ can be compared with a preference probability.
This is similar to the dueling bandit setting \cite{IF_2012}, where each edge is an arm, and we can compare the arms (edges) to find the best arm (edge), i.e., finding the best candidate for a position.
Thus, from now on, we will use arms and edges interchangeably.
We define $K = \sum_{j=1}^{\ell} \frac{|E_j|(|E_j|-1)}{2}$, which is the number of all possible duels between any two comparable arms.

We assume that there is at least one matching with cardinality $\ell$ in $G$, meaning that we can find at least one candidate for each position without a conflict. 
The {\em decision class} $\cM\subset 2^E$ is the set of all maximum matchings in $G$.
We can also view a matching as a team that specifies which candidate shall play which position for all the positions.
Given a matching $M$ and a position $s_j$, we use $e(M,j)$ to represent the edge in $M$ that connects to position $s_j$.
For any two matchings $M_1,M_2\in \cM$, we define the preference probability of $M_1$ over $M_2$ as follows:

\begin{align}
    f(M_1,M_2,P) := \frac{1}{\ell}\sum_{j=1}^{\ell}p_{e(M_1, j),e(M_2,j)}. \label{eq:define_f}
\end{align}
It is easy to show that $f(M_1,M_2,P) = 1 - f(M_2,M_1,P)$.
Written in vector representation, we have $f(M_1,M_2,P) = \frac{1}{\ell}\chi_{M_1}^T\cdot P\cdot \chi_{M_2}$.

Now, we define the ``best'' matching in the decision class $\cM$. There are several different definitions, e.g., Borda winner \cite{emerson2013original,emerson2016majority}, Condorcet winner \cite{black1948rationale}, and Copeland winner \cite{copeland1951reasonable,saari1996copeland}. 
In this paper, we focus on the Borda winner and the Condorcet winner, the definitions of which are given below.

\paragraph{Borda winner} The Borda winner refers to the winner that maximizes the average preference probability over the decision class, which we call ``Borda score''. Mathematically, in our framework, the Borda score of any matching $M_x \in \cM$ and the Borda winner are
defined as:
\begin{align}
 B(M_x)= &  \frac{1}{|\mathcal{M}|} \sum_{M_y\in\cM}f(M_x, M_y, P), \label{eq:def_borda_score}
\\
 M^B_*= & \argmax \limits_{M_x\in\cM} B(M_x). \label{eq:def_borda_winner}
\end{align}
For the pure exploration task, we assume that there is a unique Borda winner, similar to the assumption in other pure exploration tasks \cite{Elim2006_explore1,SAR_bubeck2013multiple,cpe_icml14,nearly_ChenLJ2017}.
Figure~\ref{fig:borda-winner} shows the Borda winner as the matching with the red edges, because according to the preference matrix in Figure~\ref{tab:probability}, it has the largest Borda score of $0.64$.

\paragraph{Condorcet winner} The Condorcet winner is the matching that always wins when compared to others. In our framework, the Condorcet winner is defined as the matching $M_*^C$ such that $f(M_*^C,M, P) \ge \frac{1}{2}$ for any matching $M\in \cM$. 
We assume that the Condorcet winner exists as several previous works \cite{RUCB_2014,komiyama2015regret,weak_regret2017} do, and the Condorcet winner wins over any other matching with probability strictly better than $\frac{1}{2}$, i.e. $f(M_*^C,M,P) > \frac{1}{2}$ for any $M \in \cM \setminus \{M_*^C\}$.
Figure~\ref{fig:cond-winner} shows the Condorcet winner as the matching with the red edges. It is different from the Borda winner in this example, since this matching wins over all other matchings, but its average winning score (the Borda score) $0.615$ is not as good as the Borda winner.

Our goal is to find the best matching (Borda winner or the Condorcet winner) by exploring the duels at all the positions, and we want  the number of duels that we need to explore as small as possible.
This is the problem of combinatorial pure exploration for dueling bandits (CPE-DB).
More precisely, at the beginning, the graph $G(C,S,E)$ is given to the learner, but the preference matrix $P$ is unknown. 
Because the learner does not know the preference probability for arms connected to the same position, she needs to sample the duel between edges. 
In each round the learner samples one duel pair $(e,e')$ for some position, and she observes a Bernoulli random variable $X_{e,e'}$ with $\Pr\{X_{e,e'} = 1\} = p_{e,e'}$. 
The observed feedback could be used to help to select future pairs to sample.
Our objective is to find the Borda winner $M^B_*$ or the Condorcet winner $M_*^C$ with as few samples as possible. 

\section{Efficient Exploration for Borda Winner}\label{sec:borda}
In this section, we first show the reduction of the Borda winner identification problem to the combinatorial pure exploration for multi-armed bandit (CPE-MAB) problem, originally proposed and studied in~\cite{cpe_icml14}. Next, we introduce an efficient PAC pure exploration algorithm $\Cborda$ for Borda winner, and show that with an almost uniform sampler for perfect matchings \cite{uniform_sample_perfect_matching}, $\Cborda$ has both tight sample complexity and fully-polynomial time complexity. 
Then, based on the PAC algorithm $\Cborda$, we further propose an exact pure exploration algorithm $\CbordaExact$ for Borda Winner, and provide its sample complexity upper bound.
Finally, we present the sample complexity lower bound for identifying the Borda winner.

\subsection{Reduction to Conventional Combinatorial Pure Exploration}
In order to show the reduction of CPE-DB for  Borda winner to the conventional  CPE-MAB \cite{cpe_icml14} problem, we first define the rewards for edges. Then, we define the reward of a matching to be the sum of its edge rewards.
Based on the reward definitions, it can be shown that the problem of identifying the Borda winner is equivalent to identifying the matching with the maximum reward. Specifically, for any edge $e=(c_i,s_j) \in E$ and matching $M\in\cM$,  we define their rewards and the reduction relationship between the two problems as follows:
\begin{align}
    w(e)= & \frac{1}{|\mathcal{M}|} \sum_{M\in\cM} p_{e,e(M,j)}, \nonumber
\\
     w(M) =  & \sum_{e \in M} w(e) \overset{(a)}{=} \ell \cdot B(M), \label{eq:reduction}
\\
 M^B_*= & \argmax \limits_{M\in\cM} B(M) = \argmax \limits_{M\in\cM} w(M), \nonumber
\end{align}
where the equality (a) is due to the definitions of the Borda score (Eq. \eqref{eq:def_borda_score}) and preference probability between two matchings (Eq. \eqref{eq:define_f})
\OnlyInShort{(see the supplementary material for a detailed proof of equality (a)).}
\OnlyInFull{(see Appendix~\ref{sec:appendix_borda_reduction} for a detailed proof of equality (a)).}

It remains to show how to efficiently learn the reward $w(e)$ for edge $e$, by sampling arm pairs in CPE-DB.

First, we can see that for any edge $e=(c_i,s_j)$, $w(e)$ is exactly the {\em expected} preference probability of $e$ over $e(\bar{M},j)$, 
where $\bar{M}$ is a uniformly sampled  matching from $\cM$. 
In other words, we could treat $e$ as a base arm in the CPE-MAB setting with mean reward $w(e)$, and we could obtain an unbiased sample for $e$ if we can uniformly sample $\bar{M}$ from $\mathcal{M}$ and then play the duel $(e, e(\bar{M},j))$ to observe the outcome.
However, a naive sampling method on $\mathcal{M}$  would take exponential time. 
To resolve this issue,  we employ a fully-polynomial almost uniform sampler for perfect matchings \cite{uniform_sample_perfect_matching} $\mathcal{S}(\eta)$  to  obtain an almost uniformly sampled matching $M'$ from $\cM$. Below we give the formal definition of $\mathcal{S}(\eta)$.
\begin{definition} \label{def:sampler}
An almost uniform sampler for perfect matchings is a
    randomized algorithm $\mathcal{S}(\eta)$ that, if given any bipartite graph $G$ and bias parameter $\eta$, it returns a random perfect matching from a distribution $\pi'$ that satisfies
    $$
    d_{\textup{tv}}(\pi', \pi)= \frac{1}{2} \sum \limits_{x \in \Theta } |\pi'(x) - \pi(x)| \leq  \eta,
    $$
    where $d_{\textup{tv}}$ is the total variation, $\Theta$ is the set of all perfect matchings in $G$ and $\pi$ is the uniform distribution on $\Theta$.
\end{definition}

Next, we show how to obtain $M'$  using  $\mathcal{S}(\eta)$. We add some ficticious vertices in $S$ and ficticious edges in $E$ to construct a new bipartite graph $G'(C,S',E')$ where $|C|=|S'|$. There is a one-to-$n$ relationship between a maximum matchings in  $G$ and a perfect matchings in  $G'$.
Then, with $\mathcal{S}(\eta)$, we can almost uniformly sample a maximum matching $M'$ from $G$  in  fully-polynomial time.
\OnlyInShort{We defer the details for sampling with $\mathcal{S}(\eta)$ to the supplementary material.}
\OnlyInFull{We defer the details for sampling with $\mathcal{S}(\eta)$ to Appendix~\ref{sec:appendix_borda_sampler}.}

\subsection{Efficient PAC Pure Exploration Algorithm}
In the previous subsection, we present a reduction of CPE-DB for  Borda winner to the conventional  CPE-MAB \cite{cpe_icml14} problem. However, directly applying the existing $\Clucb$ algorithm in \cite{cpe_icml14}  cannot obtain an efficient algorithm for our problem. The main obstacle is that there is currently no efficient algorithm to sample from an exact uniform distribution over all the maximum matchings in a general bipartite graph, and thus the original $\Clucb$ algorithm is not directly applicable. To tackle this problem, we need to use an approximate sampler and modify the original $\Clucb$ algorithm to handle the bias introduced by the approximate sampler.


Algorithm \ref{alg:borda-online} illustrates an efficient PAC pure exploration algorithm $\Cborda$ for the Borda winner case. Given a confidence level $\delta$ and an accuracy requirement $\varepsilon$, $\Cborda$ returns an approximate Borda winner $\Out$ such that $B(\Out) \geq B(M^B_*) -\varepsilon$  with probability at least $1-\delta$.  

$\Cborda$ is built on the $\Clucb$ \cite{cpe_icml14} algorithm designed for the conventional CPE-MAB problem, and  $\Cborda$ efficiently transforms the original numerical observations to the equivalent relative observations.  
In particular, the maximization oracle $\MaxMatchingOracle(\cdot)$ called in $\Cborda$ is exactly the maximum-weighted maximum-cardinality matching algorithm, performed in fully-polynomial time.  
The main structure follows the $\Clucb$ algorithm: in each round, we first use the empirical mean $\bar{\boldsymbol{w}}_t$ as the input to 
    the oracle $\MaxMatchingOracle(\cdot)$ to find a matching $M_t$.
Then we use the lower confidence bounds for all edges in $M_t$ and upper confidence bounds for all edges outside $M_t$
    as the input and call $\MaxMatchingOracle(\cdot)$ again to find an adjusted matching $\tilde{M}_t$.
If the difference in weights of the adjusted and non-adjusted matchings are small (line~\ref{line:ifcheck}), the algorithm stops and returns $M_t$ as the final matching.
If not, the algorithm finds the edge $z_t$ in the symmetric difference of $\tilde{M}_t$ and $M_t$. Then, the algorithm samples a matching $M'$ using sampler $\mathcal{S}(\eta)$,
and plays a duel between $z_t$ and the corresponding edge in $M'$ with the same position as $z_t$. After playing the duel, the algorithm observes the result and updates empirical mean $\bar{w}_{t+1}(z_t)$.
With the fast maximization oracle $\MaxMatchingOracle(\cdot)$ and  sampler $\mathcal{S}(\eta)$, the $\Cborda$ algorithm  can be performed in fully-polynomial time.

\begin{algorithm}[!th]
    \caption{$\Cborda$}
    \label{alg:borda-online}
    \begin{algorithmic}[1]
    \STATE {\bfseries Input:} confidence  $\delta$, accuracy   $\varepsilon$, bipartite graph $G$, maximization oracle $\MaxMatchingOracle(\cdot)$: $\mathbb{R}^m \rightarrow \mathcal{M}$ and  almost uniform sampler  for perfect matchings $\mathcal{S}(\eta)$
    \STATE Set bias parameter $\eta \leftarrow \frac{1}{8}\varepsilon$ 

    \STATE Initialize $T_{1}(e) \leftarrow 0$ and $\bar{w}_{1}(e) \leftarrow 0$ for all $e \in E$ 
    \FOR  {$t=1,2,...$}
        \STATE $M_t \leftarrow \MaxMatchingOracle(\bar{\boldsymbol{w}}_t)$ 
        \STATE Compute confidence radius $c_t(e) \leftarrow \sqrt{\frac{\ln (\frac{4K t^3}{\delta})}{2 T_{t}(e)}}$ for all $e \in E$ \quad //  $\frac{x}{0}:=1$ for any $x$
        \FOR{ \textup{all } $e \in E$}
            \IF { $e \in M_t$ } 
                \STATE $ \tilde{w}_t(e) \leftarrow \bar{w}_t(e) - c_t(e) -\frac{1}{4}\varepsilon $ 
            \ELSE
                \STATE $ \tilde{w}_t(e) \leftarrow \bar{w}_t(e) + c_t(e) + \frac{1}{4}\varepsilon $ 
            \ENDIF \quad     // $\bar{w}_t(e):=0$ if $T_{t}(e)=0$
        \ENDFOR
        \STATE $\tilde{M}_t \leftarrow \MaxMatchingOracle (\tilde{\boldsymbol{w}}_t)$ 
        \IF{ $\tilde{w}_t(\tilde{M}_t) - \tilde{w}_t(M_t) \leq \ell \varepsilon$ \label{line:ifcheck}}
            \STATE $\Out \leftarrow M_t$ 
            \STATE  {\bfseries return } $\Out    $ 
        \ENDIF
        \STATE $z_t \leftarrow \mathop{\arg \max}_{ e \in (\tilde{M}_t \setminus M_t) \cup  ( M_t \setminus \tilde{M}_t)} c_t(e)$ 
        \STATE Sample a matching $M'$ from $\mathcal{M}$ using $\mathcal{S}(\eta)$ \label{line:borda_uniform_sample}
        \STATE Pull the duel $(z_t, e')$, where $e'=e(M', s(z_t))$
        \STATE Update $\bar{w}_{t+1}(z_t) \leftarrow \frac{\bar{w}_{t}(z_t) \cdot T_{t}(z_t) + X_{t}(z_t)}{T_{t}(z_t) +1}$ where $X_{t}(z_t)$ takes value $1$ if $z_t$ wins, $0$ otherwise, and $T_{t+1}(z_t) \leftarrow T_{t}(z_t) +1$ \label{alg_borda_pac_line_X_1}
    \ENDFOR
\end{algorithmic}
\end{algorithm}

To formally state the sample complexity upper bound of the $\Cborda$ algorithm, we need to first define  the width of $G$, the Borda gap and the Borda hardness.

\begin{restatable}[Width]{definition}{defwidth}\label{def:width}
For a bipartite graph $G$, let $\cM(G)$ denote the set of all its maximum matchings. For any $M_1, M_2 \in \cM(G)$ such that $M_1 \neq M_2$, we define $\Width(M_1,M_2)$ as the number of edges of the maximum connected component in their union graph. Then, we define the width of bipartite graph $G$ as
$$\Width(G)=\max_{ \begin{subarray}{c}M_1, M_2 \in \cM(G)\\ M_1 \neq M_2 \end{subarray}} \Width(M_1,M_2).$$
\end{restatable}

This width definition for bipartite maximum matching is inline with the general width definition in \cite{cpe_icml14}. 
\OnlyInFull{We establish the equivalency between our width definition for bipartite maximum matching and that in \cite{cpe_icml14}, and defer the proof to Appendix~\ref{sec:appendix_borda_width}.}

\begin{restatable}[Borda gap]{definition}{}\label{def:gap-borda}
We define the Borda gap $\Delta^B_e$ for any edge $e \in E$ as 
$$
    \Delta^B_e=\left\{\begin{matrix}
    w(M^B_*)- \max \limits_{M \in \mathcal{M}: e \in M} w(M)  \text{ \quad if } e \notin M_*,
    \\ 
    w(M^B_*)- \max \limits_{M \in \mathcal{M}: e \notin M} w(M)  \text{ \quad if } e \in M_*,
    \end{matrix} \right.
$$
where we make the convention that the maximum value of an empty set is $-\infty$.
\end{restatable}

\begin{restatable}[Borda hardness]{definition}{}\label{def:hardness-borda}
We define the hardness $H^B$ for identifying Borda winner in CPE-DB as
\[
    H^B:= \sum \limits_{e \in E} \frac{1}{(\Delta^B_e)^2}.
\]
\end{restatable}
The Borda gap and Borda hardness definitions are naturally inherited from those in \cite{cpe_icml14}. For each edge $e \notin M_*^B$, the Borda gap $\Delta^B_e$ is the sub-optimality of the best matching that includes edge $e$, while for each edge $e \in M_*^B$ the Borda gap $\Delta^B_e$ is the sub-optimality of the best matching that does not include edge $e$. The Borda hardness $H^B$ is the sum of inverse squared Borda gaps, which represents the problem hardness for identifying the Borda winner.

Now we present a problem-dependent upper bound of the sample complexity for the $\Cborda$ algorithm.

\begin{restatable}[$\Cborda$]{theorem}{thmcborda}\label{thm:borda_ub}
With probability at least $1-\delta$, the $\Cborda$ algorithm (Algorithm  \ref{alg:borda-online}) returns an approximate Borda winner $\Out$ such that $B({\sf Out}) \geq B(M^B_*) -\varepsilon$ with sample complexity
\[
O \left( H^B_{\varepsilon} \ln \left ( \frac{H^B_{\varepsilon}}{\delta}   \right)  \right),
\]
where $H^B_{\varepsilon}:=\sum_
{e \in E} \min \left\{ \frac{\Width(G)^2}{(\Delta^B_e)^2}, \frac{1}{\varepsilon^2} \right\}$.
\end{restatable}

We can see that when the accuracy parameter $\varepsilon$ is small enough, $ H^B_{\varepsilon}$ coincides with the hardness
    metric $H^B$.
\OnlyInFull{We defer the detailed proof of Theorem \ref{thm:borda_ub} to Appendix~\ref{sec:appendix_borda_ub}.}

\subsection{Efficient Exact Pure Exploration Algorithm}
Based on the PAC  algorithm $\Cborda$, we further design an efficient exact pure exploration algorithm $\CbordaExact$ for Borda winner and analyze its sample complexity upper bound.
Generally speaking, $\CbordaExact$ performs $\Cborda$ as a sub-procedure, and guesses the smallest Borda gap $\Delta^B_{\textup{min}}:= \min_{e \in E}  \Delta^B_e $. Iterating epoch $q=1,2,\dots$, we set accuracy $\varepsilon_q = \frac{1}{2^q}$ and confidence $\delta_q = \frac{\delta}{2q^2}$. $\CbordaExact$ will guess $ \Delta^B_{\textup{min}} > \ell  \varepsilon_q $, and call  $\Cborda$ as a sub-procedure with parameters $\varepsilon_q$, $\delta_q$. 
If the adjusted matching $\tilde{M}_t$ has exactly the same weight as the non-adjusted matching $M_t$ 
    ($\tilde{w}_t(\tilde{M}_t) = \tilde{w}_t(M_t)$, similar as in line~\ref{line:ifcheck} of Algorithm~\ref{alg:borda-online}), 
    then the algorithm stops and returns $M_t$ as the final matching.
If $\tilde{w}_t(\tilde{M}_t) \ne \tilde{w}_t(M_t)$ but they differ within $\ell \varepsilon_q$, then
    the current epoch stops and $\CbordaExact$ will enter the next epoch 
    and cut the guess in half ($\varepsilon_{q+1} = \varepsilon_q/2$)
    (\OnlyInFull{See Appendix~\ref{sec:appendix_borda_exact} for the algorithm pseudocode}.\OnlyInShort{Please see the supplementary materials for more details.})
Using this technique, we can obtain an algorithm to identify the exact Borda winner with a loss of logarithmic factors in its sample complexity upper bound. 

\OnlyInShort{Below we present a problem-dependent upper bound of the sample complexity for the $\CbordaExact$ algorithm and defer the detailed algorithm and proof to the appendix.}
\OnlyInFull{Below we present a problem-dependent upper bound of the sample complexity for the $\CbordaExact$ algorithm and defer the detailed algorithm and proof to Appendix~\ref{sec:appendix_borda_exact}.}

\begin{restatable}[$\CbordaExact$]{theorem}{thmcbordaexact}\label{thm:borda_ub_exact}
\OnlyInFull{With probability at least $1-\delta$, the $\CbordaExact$ algorithm (Algorithm  \ref{alg:borda_exact}) returns the Borda winner  with sample complexity}
\OnlyInShort{With probability at least $1-\delta$, the $\CbordaExact$ algorithm returns the Borda winner  with sample complexity}
\begin{small}
\begin{align*}
    & O \Bigg( \Width(G)^2 H^B \cdot \ln \left(\frac{\ell}{\Delta^B_{\textup{min}}} \right) \cdot\\
    & \quad \Bigg(\ln  \left (  \frac{\Width(G)  H^B}{\delta}   \right) 
     + \ln \ln \left(\frac{\ell}{\Delta^B_{\textup{min}}} \right) \Bigg) \Bigg),     
\end{align*}
\end{small}
where $\Delta^B_{\textup{min}}:= \min \limits_{e \in E}  \Delta^B_e $.
\end{restatable}

\subsection{Lower Bound}

To formally state our result for lower bound, we first introduce the definition of $\delta$-correct algorithm as follows. For any $\delta \in (0,1)$, we call an algorithm $\mathbb{A}$ a $\delta$-correct algorithm if, for any problem instance of CPE-DB with Borda winner, algorithm $\mathbb{A}$ identifies the Borda winner with probability at least $1-\delta$.

Now we give a problem-dependent lower bound on the sample complexity for CPE-DB with Borda winner.
\begin{restatable}[Borda lower bound]{theorem}{thmbordalb} \label{thm:borda_lb}
    Consider the problem of combinatorial pure exploration for identifying the Borda winner. Suppose that, for some constant $ \gamma \in (0, \frac{1}{4})$,  $\frac{1}{2}-\gamma \leq p_{e_i,e_j} \leq \frac{1}{2}+\gamma , \  \forall e_i,e_j \in E$ and $\frac{|\mathcal{M}|}{|\mathcal{M}| - |\cM_e|} \leq \frac{1-4\gamma}{4 \gamma \ell}, \  \forall e \in E$. Then, for any $\delta \in (0,0.1)$, any $\delta$-correct algorithm has sample complexity 
    $
    \Omega \left (H^B \ln \Big( \frac{1}{\delta} \Big) \right),
    $
    where  $\cM_e:=\{ M \in \cM: e \in M\}$.
\end{restatable}

\OnlyInFull{We defer the detailed proof of Theorem \ref{thm:borda_lb} to Appendix~\ref{sec:appendix_borda_lb}.}
\OnlyInShort{We defer the detailed proof of Theorem \ref{thm:borda_lb} to the appendix.}

From the upper bounds (Theorems \ref{thm:borda_ub},\ref{thm:borda_ub_exact}) and lower bound (Theorem \ref{thm:borda_lb}), 
we see that when ignoring the logarithmic factors, our algorithms are tight on the hardness metric $H^B$.
However, whether the $\Width(G)$ factor is tight or not remains unclear and we leave it for future investigation.


\section{Efficient Exploration for Condorcet Winner}\label{sec:condorcet}
In this section, we introduce the efficient pure exploration algorithm $\Carcond$ to find a Condorcet winner. We first introduce the efficient pure exploration part assuming there exists ``an oracle'' that performs like a black-box, and we show the correctness and the sample complexity of $\Carcond$ given the oracle. Next, we present the details of the oracle and show that the time complexity of the oracle is polynomial. Then, we apply the verification technique \cite{verification_karnin2016} to improve our sample complexity further. Finally, we give the sample complexity lower bound for finding the Condorcet winner.

\subsection{Efficient Pure Exploration Algorithm: $\Carcond$}
We first introduce our algorithm $\Carcond$ for CPE-DB for the Condorcet winner assuming that there is a proper ``oracle''. Note that finding the Condorcet winner if existed is equivalent to the following optimization problem,
\[\max_{x = \chi_{M_1}}\min_{y=\chi_{M_2}}\frac{1}{\ell}x^TPy,\]
where $M_1,M_2\in\cM$ are feasible matchings and the value is optimal when $x = y = \chi_{M^C_*}$.
This is  because if $M_1$ is not the Condorcet winner $M^C_*$, it will lose to $M^C_*$ with score $\chi_{M_1}^T P \chi_{M^C_*}<1/2$, and only when
    $x = \chi_{M^C_*}$, $\min_{y=\chi_{M_2}}\frac{1}{\ell}x^TPy$ reaches $1/2$ when $y = \chi_{M^C_*}$.
However, the optimization problem is ``discrete'' and we first use the continuous relaxation technique to solve the following optimization problem
\begin{equation} \label{eq:maxmin}
    \max_{x \in\cP(\cM)}\min_{y\in\cP(\cM)}\frac{1}{\ell}x^TPy,
\end{equation}
where $\cP(\cM) = \{\sum_i\lambda_i\chi_{M_i}:M_i\in\cM,\sum_{i}\lambda_i = 1,\lambda_i \ge 0\}$ is the convex hull of the vectors $\chi_M,M\in\cM$. 
There is an algorithm that can solve $x,y$ approximately in polynomial time, but solving the optimization problem of Eq.~\eqref{eq:maxmin} is not enough for our CPE-DB problem. Therefore, we need the following more powerful oracle.

We assume that there is an oracle $\Oracle_{\varepsilon}$ that takes the inputs $\varepsilon,A_1,R_1,A_2,R_2,Q$, where $\varepsilon$ is the error of the oracle, $A_1,R_1,A_2,R_2\subset E$ and $Q\in [0,1]^{m\times m}$. The oracle can approximately solve the following optimization
\begin{equation}\label{equ:condorcet-oracle-optimization}
    \max_{x\in \cP(\cM,A_1,R_1)}\min_{y\in\cP(\cM,A_2,R_2)}\frac{1}{\ell}x^TQy,
\end{equation}
where $\cP(\cM, A, R) = \{\sum_i\lambda_i\chi_{M_i}:M_i\in\cM,A\subset M_i, R\subset (M_i)^c,\sum_{i}\lambda_i = 1,\lambda_i \ge 0\}$ 
is the convex hull of the vector representations of the matchings, 
such that all edges in $A$ are included in the matching and none of the edges in $R$ is included in the matching. 
More specifically, we assume that the oracle $\Oracle_{\varepsilon}$ will compute a solution $x_0$ 
    that satisfies both the constraint and the following guarantee:
\begin{align*}
    &\min_{y\in\cP(\cM,A_2,R_2)}\frac{1}{\ell}x_0^TQy\\
    \ge& \max_{x\in \cP(\cM,A_1,R_1)}\min_{y\in\cP(\cM,A_2,R_2)}\frac{1}{\ell}x^TQy - \varepsilon.
\end{align*}

In the algorithm, we only require that the oracle $\Oracle_{\varepsilon}$ returns the value
    $\min_{y\in\cP(\cM,A_2,R_2)}\frac{1}{\ell}x_0^TQy$, not the $x_0$.

Given the oracle $\Oracle_{\varepsilon}$, the high level idea of $\Carcond$ (Algorithm \ref{alg:condorcet-online}) is as follows: 
If we know how to set the approximation parameter properly, then in every round we partition the edge set $E$ into $A$, $R$, and $U$, 
    where $A$ is the set of the edges that should be included in the Condorcet winner, 
    $R$ is the set of edges that should be excluded, and $U$ are the remaining undecided edges. 
In each round, we only sample the duel between two comparable edges in the set $U$ (Line \ref{line:sample-u}). 
Then, we use the upper and lower confidence bounds to estimate the real preference matrix $P$ (Line \ref{line:upper-lower-confidence}). 
After that, for every undecided edge $e$, we enforce it to be included in the optimal solution or to be excluded in the solution, and use the oracle to see if the included and excluded cases vary much. If so, we classify edge $e$ into $A$ or $R$ in the next round (Line \ref{line:included-excluded}). Since we do not know how to set the approximation parameter properly, we use the ``doubling trick'' to shrink the 
    approximation parameter $\varepsilon_q$ by a factor of $2$ in each epoch $q$ (Line \ref{line:shrink}).

For the value of the confidence radius and the upper and lower confidence bound for the matrix $P$, we use the following quantity for the confidence radius of the winning probability of the duel between any two comparable arms.
\begin{equation}\label{equ:cr-carcond}
    c_t(e_i,e_j) = \sqrt{\frac{\ln(4Kt^3/\delta)}{2T_t(e_i,e_j)}},
\end{equation}
where $T_t(e_i,e_j)$ is the number of duels between two comparable arms $e_i,e_j$ at the beginning of round $t$. Now given some duels (at least one) between $e_i,e_j$, we define $\hat p_t(e_i,e_j)$ as the empirical winning probability of $e_i$ over $e_j$ \emph{up to} round $t$'s exploration phase, and we define
\begin{align}\label{equ:lucb-carcond}
    \bar p_t(e_i,e_j) :=& \min\{1,\hat p_t(e_i,e_j) + c_t(e_i,e_j)\},\\
    \underline{p}_t(e_i,e_j) :=& \max\{0, \hat p_t(e_i,e_j) - c_t(e_i,e_j)\}.\nonumber
\end{align}
$\bar p_t(e_i,e_j)$ and $\underline{p}_t(e_i,e_j)$ can be interpreted as the upper and lower confidence bounds of the winning probability of $e_1$ over $e_2$. Then we denote $\bar P_t$ as the matrix where $\bar P_{t,ij} := \bar p_t(e_i,e_j)$ where $i,j$ are edge indices, $\bar P_{t,ii} := 0.5$, and $\bar P_{t,ij} = 0$ for any 2 incomparable indices. Similarly, we   define $\underline{P}_t$ as the matrix where $\underline{P}_{t,ij} := \underline{p}_t(e_i,e_j)$, $\underline{P}_{t,ii} := 0.5$, and $\underline{P}_{t,ij} = 0$ for any two incomparable indices.

\begin{algorithm}[t]
\caption{$\Carcond$}
\label{alg:condorcet-online}
\begin{algorithmic}[1]
    \STATE {\bfseries Input:} Bipartite graph $G$, Oracle $\Oracle_{\varepsilon}$ with accuracy $\varepsilon$
    \STATE $A_0\leftarrow \phi, R_0\leftarrow \phi, U_0\leftarrow E,e_0 = 0$.
    \FOR{$q = 1,2,\dots$}
    \STATE $\varepsilon_q\leftarrow \frac{1}{2^q},e_q\leftarrow \frac{1}{\varepsilon_q^2}$\label{line:shrink}
    \FOR{$t = e_{q-1}+1,e_{q-1}+2,\dots,e_{q}$}
        \STATE For every $e_1\neq e_2$ and $e_1,e_2\in E_j$ for some $j$ and $e_1,e_2\in U_{t-1}$, sample duel between $e_1,e_2$\label{line:sample-u}
        \STATE Compute $\bar P_t, \underline{P}_t$\label{line:upper-lower-confidence}
        \STATE $A_{t}\leftarrow A_{t-1},R_{t}\leftarrow R_{t-1}, U_t \leftarrow U_{t-1}$
        \FOR{$e\in U_{t-1}$}\label{line:included-excluded}
            \STATE // We use $A,R$ as shorthands for $A_{t-1},R_{t-1}$
            \STATE $\text{InU}=\Oracle_{\varepsilon_q}(A\cup \{e\}, R, A, R, \bar P_t)$
            \STATE $\text{InL}=\Oracle_{\varepsilon_q}(A\cup \{e\}, R, A, R, \underline{P}_t)$
            \STATE $\text{ExU}=\Oracle_{\varepsilon_q}(A, R\cup \{e\}, A, R, \bar P_t)$
            \STATE $\text{ExL}=\Oracle_{\varepsilon_q}(A, R\cup \{e\}, A, R, \underline{P}_t)$
            \IF{$\text{InL} > \text{ExU} + \varepsilon_q$}
                \STATE $A_t\leftarrow A_{t}\cup \{e\},U_t\leftarrow U_{t}\setminus \{e\}$
            \ELSIF{$\text{ExL} > \text{InU} + \varepsilon_q$}
                \STATE $R_t\leftarrow R_{t-1}\cup \{e\},U_t\leftarrow U_{t}\setminus \{e\}$
            \ENDIF
            \STATE \algorithmicif\ $|A_t| = \ell$ \algorithmicthen\ $\Out \leftarrow A$, {\bfseries return } $\Out$
        \ENDFOR
    \ENDFOR
    \ENDFOR
\end{algorithmic}
\end{algorithm}

\paragraph{Sample complexity for $\Carcond$}
To present our main result on the sample complexity of $\Carcond$, we need to first introduce the notion of {\em gap} for each edge and each comparable pair under the Condorcet setting.

\begin{restatable}[Condorcet gap]{definition}{defgapcarcond}\label{def:gap-carcond}
    We define the Condorcet gap $\Delta^C_e$ of an edge $e$ as the following quantity.
    \[
        \Delta^C_e = \left\{\begin{aligned}
            &1/2 - \max_{\chi_M,e\in M}\frac{1}{\ell}\chi_M^T P\chi_{M_*^C}, & \text{if }e\notin M_*^C\\
            &1/2 - \max_{\chi_M,e\notin M}\frac{1}{\ell}\chi_M^T P\chi_{M_*^C}, & \text{if }e\in M_*^C
        \end{aligned}\right.
    \]
    Then we define the gap $\Delta^C_{e,e'}$ for a pair of arms $e\neq e'$ and $e,e'\in E_j$ as the following quantity
    $\Delta^C_{e,e'} = \max\{\Delta^C_{e},\Delta^C_{e'}\}$.
\end{restatable}

The definition of gap is very similar to the gap defined in \cite{cpe_icml14}. Intuitively speaking, the definition of the gap of each edge $e$ is a measurement of how easily $e$ will be classified into the accepted set $A$ or the rejected set $R$. Given the definition of the gap, we have the following main theorem for the Condorcet setting. 

\begin{restatable}[$\Carcond$]{theorem}{thmcarcond}\label{thm:carcond}
    With probability at least $1-\delta$, algorithm $\Carcond$ returns the correct Condorcet winner with a sample complexity bounded by
    \[
        O\left(\sum_{j=1}^{\ell}\sum_{ e\neq e',e,e'\in E_j}\frac{1}{(\Delta^C_{e,e'})^2}\ln\left(\frac{K}{\delta(\Delta^C_{e,e'})^2}\right)\right).
    \]
\end{restatable}

Generally speaking, our algorithm sequentially classifies each edge into $M^C_*$ or $(M^C_*)^c$. The definition of the gap shows the sub-optimality of wrongly classifying each edge, and $\frac{1}{(\Delta^C_e)^2}$ is roughly the number of times to correctly classify the edge $e$. Because each query is a sample between two edges $e,e'$, the number of query between $e,e'$ is roughly $1/(\Delta^C_{e,e'})^2$, this is so as when we correctly classify an edge, we will not need to query any pair that contains this edge. Summing over all comparable pairs of edges, we get our upper bound when omitting all logarithm terms.

When there is only one position, our problem reduces to the original dueling bandit problem. 
In special cases when the Condorcet winner beat every arm with the largest margin (formally, for all arm $i\in [m]$, $i^C = \arg\max_{j\in [m]}\Pr\{j\text{ wins }i\}$), 
our sample complexity bound is at the same order as the state-of-the-art \cite{verification_karnin2016} when omitting the logarithmic terms.

\subsection{Implementation of Oracle}
In this part, we present the high level idea of our method to solve the optimization problem (\cref{equ:condorcet-oracle-optimization}).
If we define
\[g(x) = \min_{y\in\cP(\cM,A_2,R_2)}\frac{1}{\ell}x^TQy,\]
then $g$ is concave in $x$, since $x^TQy$ is linear in $x$ and the minimum of linear functions is a concave function. 
Also note that the constraint set $\cP(\cM,A_2,R_2)$ is a convex set since it is defined as the convex hull of the vector representations. 
Thus, using the projected sub-gradient ascent method, 
    we can solve the optimization problem by an error of $\varepsilon$ in $O(\frac{1}{\varepsilon})$ number of iterations. 
To do so, we need to address two problems:
    how to compute the gradient at a given point, and how to compute the projection efficiently.

The first problem is rather easy to solve, because if we want to compute the sub-gradient at a given point $x_0$, 
it suffices to compute the parameter $y_0 = \arg\min_{y\in\cP(\cM,A_2,R_2)}x_0^TQy$, and the sub-gradient will be $\frac{1}{\ell}Qy_0\in \partial_x g(x_0)$. 
Computing the parameter $y_0$ can be done in polynomial time, since the minimum cost maximum matching can be solved in polynomial time.

The second problem is the main challenge. 
Note that there may be an exponentially large number of vertices in the polytope $\cP(\cM,A_2,R_2)$ because the number of feasible matchings may be exponential, 
    and we cannot solve the projection step in general. 
However, if we can tolerate some error in the projection step, we may solve the approximate projection in polynomial time by the Frank-Wolfe algorithm. Then, we can set the approximate projection error to be relatively small, so the cumulative error due to the projection can also be bounded. In this way, we can solve the optimization problem \cref{equ:condorcet-oracle-optimization} with $poly(1/\varepsilon, m, K, \ell)$ time complexity.

\OnlyInFull{Please see Appendix~\ref{sec:convex-opt} for more backgrounds on projected sub-gradient ascent, Frank-Wolfe, and Appendix~\ref{sec:oracle} for the detailed implementation of the oracle.}
\OnlyInShort{Please see the supplementary materials for more backgrounds on convex optimization and the detailed implementation of the oracle.}

\subsection{Further Improvements through Verification}

\begin{algorithm}[t]
	\caption{$\Carparallel$}
	\label{alg:condorcet-parallel}
	\begin{algorithmic}[1]
		\STATE {\bfseries Input:} confidence $\delta<0.01$, algorithm $\Carverify$
		\STATE Define $\Carverify_k, k \in \mathbb{N}$ as the $\Carverify$ algorithm with confidence $\frac{\delta}{2^{k+1}}$
		\STATE Simulate $\{\Carverify_k\}_{k \in \mathbb{N}}$  in parallel
		\FOR{$t=1,2,\dots$}
		\FOR{ each $k \in \mathbb{N} \  s.t. \  t \textup{ mod } 2^k =0$}
		\STATE Start or resume  $\Carverify_k$, allowing only one sample, and then suspend $\Carverify_k$
		\IF {$\Carverify_k$ returns an answer $\Out_k$}
		\STATE $\Out \leftarrow \Out_k$
		\STATE {\bfseries return } $\Out$
		\ENDIF
		\ENDFOR
		\ENDFOR
	\end{algorithmic}
\end{algorithm}

\begin{algorithm}[h]
\caption{$\Carverify$}
\label{alg:condorcet-verify}
\begin{algorithmic}[1]
    \STATE {\bfseries Input:} confidence $\delta<0.01$, algorithm $\Carcond$
    \STATE $\delta_0 \leftarrow 0.01$
    \STATE $\hat{M}=\Carcond(\delta_0)$
    \FOR{ $t=1,2,\dots$ }
        \STATE Compute $\bar P_t, \underline{P}_t$
        \IF{$\max_{M \in \cM \setminus \{\hat{M}\} } f(M, \hat{M}, \underline{P}_t) \geq \frac{1}{2}$}
            \STATE {\bfseries return } error
        \ENDIF
        \STATE $M_t=\argmax_{M \in \cM \setminus \{\hat{M}\} }f(M, \hat{M}, \bar P_t)$
        \IF{$f(M_t, \hat{M}, \bar P_t)\leq \frac{1}{2}$}
            \STATE $\Out \leftarrow \hat{M}$ \STATE {\bfseries return } $\Out$
        \ELSE
            \STATE $(e_t, f_t) \leftarrow \argmax_{\begin{subarray}{c}  e_t \in M_t \setminus \hat{M},  f_t \in \hat{M} \setminus M_t\\ s(e_t)=s(f_t) \end{subarray}} c_t(e_t, f_t)$
            \STATE Pull the duel $(e_t, f_t)$ and update empirical means
        \ENDIF
    \ENDFOR
\end{algorithmic}
\end{algorithm}

Based on the $\Carcond$ algorithm, we further design an algorithm $\Carparallel$ for identifying Condorcet winner, which uses the parallel simulation technique \cite{parallel_ChenLJ2015,nearly_ChenLJ2017} and achieves a tighter expected sample complexity for small confidence.
$\Carparallel$ calls a variant of $\Carcond$, named $\Carverify$, which applies the verification technique \cite{verification_karnin2016} to improve the sample complexity of the original $\Carcond$.
Specifically, $\Carverify$ calls $\Carcond(\delta_0)$ to obtain a hypothesized Condorcet winner $\hat{M}$ using a constant confidence $\delta_0>\delta$. Then,  $\Carverify$ verifies the correctness of $\hat{M}$ using confidence $\delta$.
While $\Carverify$ loses a part of confidence in order to obtain better sample complexity for small confidence, $\Carparallel$ boosts the confidence to $\delta$ by simulating a sequence of $\Carverify$ in parallel and keeps the obtained better sample complexity in expectation.

Algorithm \ref{alg:condorcet-parallel} illustrates the detailed algorithm $\Carparallel$ that applies the parallel simulation technique \cite{parallel_ChenLJ2015,nearly_ChenLJ2017} and achieves a tighter expected sample complexity for small confidence.  Algorithm \ref{alg:condorcet-verify} illustrates the sub-procedure $\Carverify$ called in $\Carparallel$. $\Carverify$ is based on the original algorithm $\Carcond$ and employs the verification technique to improve the sample complexity for small confidence.

In order to formally state our result for the $\Carparallel$ algorithm, we first introduce the following definitions.

For any $e \notin M^C_*$, we define the verification gap $\tilde{\Delta}^C_{e}$ as 
\[
\min_{M \in \cM \setminus \{M^C_*\}:e \in M} \left\{ \frac{\ell}{d_{M_*^C, M}} \cdot \left(\frac{1}{2} - \frac{1}{\ell}\chi_M^T P\chi_{M_*^C} \right) \right\},
\]
where $d_{M_x, M_y}$ denotes the number of positions with different edges between $M_x$ and $M_y$, \emph{i.e.}, $d_{M_x, M_y}:= \sum_{j=1}^{\ell} \mathbb{I}\{e(M_x,j) \neq e(M_y,j)\}$.

For ease of notation, we define the following quantity
$$
H^C_{\textup{ver}}:=\sum_{e \notin M^C_*} \frac{1}{(\tilde{\Delta}^C_{e})^2}.
$$

Then, we have the main theorem of the sample complexity of algorithm $\Carparallel$.

\begin{restatable}[$\Carparallel$]{theorem}{thmparallel}\label{thm:parallel}
    Assume the existence of Condorcet winner. Then, given $\delta<0.01$, with probability at least $1-\delta$, the $\Carparallel$ algorithm (Algorithm \ref{alg:condorcet-parallel}) will return the  Condorcet winner with an expected sample complexity 
\begin{align*}
O\left(
    \sum_{j=1}^{\ell}\sum_{\begin{subarray}{c} e\neq e' \\ e,e'\in E_j\end{subarray}}\frac{\ln\left(K/(\Delta^C_{e,e'})^2\right)}{(\Delta^C_{e,e'})^2}
    + H^C_{\textup{ver}} \ln\left(\frac{H^C_{\textup{ver}}}{\delta}  \right)
    \right).    
\end{align*}
\end{restatable}

To the best of our knowledge, the best sample complexity for pure exploration of Condorcet dueling bandit is $O(n^2/\Delta^2 + n/\Delta^2\log( 1/\delta))$ by \cite{verification_karnin2016} using the verification technique. When reducing our setting to the simple Condorcet dueling bandit ($\ell=1$),  Theorem \ref{thm:parallel} recovers this result.

\OnlyInFull{We defer the detailed results and proofs to  Appendix~\ref{sec:appendix_cond_verify}.}
\OnlyInShort{We defer the detailed results and proofs to the supplementary material.}

\yihan{Add the algorithms and results of Condorcet with verification in this subsection.}

\section{Conclusion and Future Work}

In this paper, we formulate the combinatorial pure exploration for dueling bandit (CPE-DB) problem.
We consider two optimality metrics, Borda winner and Condorcet winner. For Borda winner, we first reduce the problem to CPE-MAB, and then propose efficient PAC and exact algorithms. 
We provide sample complexity upper and lower bounds for these algorithms.
For a subclass of problems the upper bound of the exact algorithm matches the lower bound when ignoring the logarithmic factor.
For Condorcet winner, we first design an FPTAS for a properly extended offline problem, 
    and then employ this FPTAS to design a novel online algorithm $\Carcond$.
To our best knowledge, $\Carcond$ is the first algorithm with polynomial running time per round for identifying the Condorcet winner in CPE-DB.

There are several promising directions worth further investigation for CPE-DB. 
One direction is to improve the sample complexity of the  $\Carcond$ algorithm without compromising its computational efficiency, and
    try to find a lower bound in this case that matches the upper bound. 
Other directions of interest include studying a more general CPE-DB model than the current candidate-position matching version, or a family of practical preference functions $f(M_1,M_2,P)$ other than linear functions. 

\section*{Acknowledgement}
The work of Yihan Du and Longbo Huang is supported in part by the National Natural Science Foundation of China Grant 61672316, the Zhongguancun Haihua Institute for Frontier Information Technology and the Turing AI Institute of Nanjing.

\bibliographystyle{icml2020}
\bibliography{ref}

\OnlyInFull{
\clearpage
\newpage
\onecolumn

\appendix
\section*{Appendix}
\section{Preliminaries}\label{sec:appendix-background}
\subsection{Maximum-Weighted Maximum-Cardinality Matching  Algorithm}
The maximum-weighted maximum-cardinality (MWMC) matching algorithm \cite{MWMC1993matching,MWMC2013gpu} is a variation of the known maximum-weighted matching algorithm. Given any bipartite graph $G$ with weighted edges, the MWMC algorithm finds the maximum-weighted matching among all maximum-cardinality matchings and operates in fully-polynomial time.

Note the the variant of MWMC, the minimum-weighted maximum-cardinality matching can also be solve efficiently. We first take the negative value of each edge and shift all of them to the positive direction, to make sure every ``new'' weight is positive. Then we call the MWMC algorithm and find the maximum-weighted maximum-cardinality matching for the new graph. Since the maximum-cardinality are the same for the 2 graphs, the MWMC solution for the new graph is the minimum-weighted maximum-cardinality matching in the original graph.

\subsection{Basic concepts and algorithms for convex optimization}\label{sec:convex-opt}
In this part, we review some basic definitions, properties, and algorithms in convex optimization. First, we give the definition of convex sets and convex functions. All of the definitions, algorithms, and properties are adapted from \cite{bubeck2015convex}.

\begin{definition}[Convex Sets and Convex functions]
    A set $\cX\subset \R^n$ is said to be convex if it contains all of its segments, i.e.
    \[\forall (x,y,\gamma)\in\cX\times\cX\times [0,1],(1-\gamma)x + \gamma y\in\cX.\]
    A function $f:\cX\to \R$ is said to be convex if $\cX$ is a convex set and
    \[\forall (x,y,\gamma)\in\cX\times\cX\times [0,1],f((1-\gamma)x + \gamma y) \le (1-\gamma)f(x) + \gamma f(y).\]
\end{definition}

The gradient of a function $f$ is a basic definition. However, there are cases when $f$ does not have gradient at every point, and we have the following definition of subgradient for convex function $f$.

\begin{definition}[Subgradients]
    Let $\cX\in\R^n$, and $f:\cX\to\R$. Then $g\in\R^n$ is a subgradient of $f$ at $x\in\cX$ if for any $y\in\cX$ one has
    \[f(x) - f(y) \le g^T(x - y).\]
    The set of subgradients of $f$ at $x$ is denoted $\partial f(x)$.
\end{definition}

Then, we have the definition of Lipschitz and Smoothness.

\begin{definition}[Lipschitz and Smoothness]
    A continuous function $f(\cdot)$ is $\ell$\text{-Lipschitz} if:
    \[\forall x_1,x_2, |f(x_1) - f(x_2)| \le \ell ||x_1 - x_2||_2\]
    A differentiable function $f(\cdot)$ is $\beta$\text{-smooth} if:
    \[\forall x_1,x_2,\ ||\nabla f(x_1) - \nabla f(x_2)||_2 \le \beta ||x_1 - x_2||_2.\]
\end{definition}

Next, we recall the definition of projection. The projection $\Pi(x,\cX)$ from a point $x\in\R^n$ to a convex set $\cX\subset\R^n$ is defined to be
\[\Pi(x,\cX) = \arg\min_{y\in\cX}||x-y||_2.\]
The projection of $x$ to $\cX$ is the point in $\cX$ that is the closest to $x$. Then, we have the following property of projection.

\begin{restatable}[Property of projection]{proposition}{propproj}\label{prop:proj}
    Let $\cX\subset \R^n$ be a convex set. For any $x\in \cX, y\in \R^n$, we have
    \[||y-x||_2 \ge ||\Pi(x,\cX) - x||_2 + ||\Pi(x,\cX)-y||_2^2.\]
\end{restatable}

The property of projection is a key lemma in the analysis of many convex optimization algorithms, including the one we use in the following sections.

Then, we briefly introduce 2 algorithms for convex optimization: Projected subgradient descent and Frank-Wolfe. We will use these 2 algorithms in our analysis.

\paragraph{Projected subgradient descent}
The projected subgradient descent acts almost the same as the projected gradient descent algorithm, except that in this case, the gradient may not exist and we use the subgradient. The projected subgradient descent algorithm iterates the following equations for $t\ge 1$:
\begin{align*}
    y^{(t+1)} =& x^{(t)} - \eta g^{(t)}, \text{where } g^{(t)}\in\partial f(x^{(t)}), \\
    x^{(t+1)} =& \Pi(y^{(t+1)},\cX)
\end{align*}
We will not directly apply the performance guarantee of the PGD algorithm, so we omit the theoretical guarantee here.

\paragraph{Frank-Wolfe Algorithm}
For a convex function $f$ defined on a convex set $\cX$, given a fixed sequence $\{\gamma_t\}_{t\ge 1}$, the Frank-Wolfe Algorithm iterate as the following for $t\ge 1$:
\begin{align*}
    y^{(t)} \in& \arg\min_{y\in \cX}\nabla f(x^{(t)})^Ty \\
    x^{(t+1)} =& (1-\gamma_t)x^{(t)} + \gamma_t y^{(t)}
\end{align*}

We have the following theoretical guarantee for Frank-Wolfe.

\begin{restatable}{proposition}{propfw}\label{prop:fw}
    Let $f$ be convex and $\beta$-smooth function with respect to norm $||\cdot||_2$, and define $D = \sup_{x,y\in\cX}||x-y||_2$, and $\gamma_s = \frac{2}{s+1}$ for $s\ge 1$. Then for any $t\ge 2$, one has
    \[f(x^{(t)}) - f(x^*) \le \frac{2\beta D^2}{t+1}.\]
\end{restatable}

\section{Omitted Proofs in Section \ref{sec:borda}}

\subsection{Reduction to Conventional Combinatorial Pure Exploration} \label{sec:appendix_borda_reduction}
In the following, we give the omitted proof of the equality (a) in Eq. \eqref{eq:reduction}.

Recall that the preference probability between two matchings $M_1,M_2 \in \cM$ are defined as
\begin{align*}
	f(M_1,M_2,P) := \frac{1}{\ell}\sum_{j=1}^{\ell}p_{e(M_1, j),e(M_2,j)}.
\end{align*}
The Borda score of any matching $M_x \in \cM$ and the Borda winner are  defined as

\begin{align*}
 B(M_x)= &  \frac{1}{|\mathcal{M}|} \sum_{M_y\in\cM}f(M_x, M_y, P) 
\\
 M^B_*= & \argmax \limits_{M_x\in\cM} B(M_x).
\end{align*}
The rewards of any edge $e=(c_e,s_j) \in E$ and any matching $M \in \cM$ are defined as 
\begin{align*}
    w(e)= & \frac{1}{|\mathcal{M}|} \sum_{M\in\cM} p_{e,e(M,j)}
\\
     w(M) =  & \sum_{e \in M} w(e) \overset{(a)}{=} \ell \cdot B(M) 
\end{align*}

Therefore, we have
\begin{align*}
     B(M_x)= &\frac{1}{|\mathcal{M}|} \sum_{M_y\in\cM} f(M_x, M_y, P) 
     \\
     = & \frac{1}{|\mathcal{M}|} \sum_{M_y\in\cM} \frac{1}{\ell} \sum_{j=1}^{\ell} p_{e(M_x, j),e(M_y,j)}
     \\
     = & \frac{1}{\ell} \sum_{j=1}^{\ell} \frac{1}{|\mathcal{M}|} \sum_{M_y\in\cM} p_{e(M_x, j),e(M_y,j)} 
     \\
     = & \frac{1}{\ell} \sum_{j=1}^{\ell} w(e(M_x, j)) 
     \\
     = & \frac{1}{\ell} \sum_{e \in M_x} w(e)
     \\
     = & \frac{1}{\ell} w(M_x),
\end{align*}
which completes the proof of the equality (a) in Eq. \eqref{eq:reduction}.

With the shown linear relationship between the Borda score of any matching and rewards of its contained edges, we can reduce  combinatorial pure exploration for  Borda dueling bandits to conventional combinatorial pure exploration.

\subsection{Details for applying the almost uniform sampler} \label{sec:appendix_borda_sampler}

\begin{figure*}[h]
\scalebox{1.1}{
\begin{minipage}[b]{.28\linewidth}
\centering

\definecolor{myblue}{RGB}{80,80,160}
\definecolor{mygreen}{RGB}{80,160,80}

\begin{tikzpicture}[thick,
  fsnode/.style={draw,circle,fill=myblue},
  ssnode/.style={draw,circle,fill=mygreen},
  vsnode/.style={draw,circle,fill=gray}
]

\begin{scope}[start chain=going below,node distance=7mm]
\foreach \i in {1,2,3,4}
  \node[fsnode,on chain] (c\i) [label=left: $c_{\i}$] {};
\end{scope}

\begin{scope}[xshift=2cm,yshift=-1cm,start chain=going below,node distance=7mm]
\foreach \i in {1,2}
  \node[ssnode,on chain] (s\i) [label=right: $s_{\i}$] {};
\end{scope}


\draw[red] (c1) -- node[sloped,above=0.1mm] {$e_1$} (s1);
\draw (c2) -- node[sloped,above=0.1mm] {$e_2$} (s1);
\draw (c3) -- node[sloped,above=0.1mm] {$e_3$} (s1);
\draw (c3) -- node[sloped,above=0.1mm] {$e_4$} (s2);
\draw[red] (c4) -- node[sloped,above=0.1mm] {$e_5$} (s2);
\end{tikzpicture}
\caption{Original bipartite graph $G$}
\label{fig:original_bigraph}
\end{minipage}
\begin{minipage}[b]{.3\linewidth}
\centering

\definecolor{myblue}{RGB}{80,80,160}
\definecolor{mygreen}{RGB}{80,160,80}

\begin{tikzpicture}[thick,
  fsnode/.style={draw,circle,fill=myblue},
  ssnode/.style={draw,circle,fill=mygreen},
  vsnode/.style={draw,circle,fill=gray}
]

\begin{scope}[start chain=going below,node distance=7mm]
\foreach \i in {1,2,3,4}
  \node[fsnode,on chain] (c\i) [label=left: $c_{\i}$] {};
\end{scope}


\begin{scope}[xshift=2cm,yshift=0cm,start chain=going below,node distance=7mm]
\foreach \i in {3}
  \node[vsnode,on chain] (s\i) [label=right: $s'_{\i}$] {};
\foreach \i in {1,2}
  \node[ssnode,on chain] (s\i) [label=right: $s_{\i}$] {};
\foreach \i in {4}
  \node[vsnode,on chain] (s\i) [label=right: $s'_{\i}$] {};
\end{scope}


\draw[red] (c1) -- node[sloped,above=0.1mm] {} (s1);
\draw (c2) -- node[sloped,above=0.1mm] {} (s1);
\draw (c3) -- node[sloped,above=0.1mm] {} (s1);
\draw (c3) -- node[sloped,above=0.1mm] {} (s2);
\draw[red] (c4) -- node[sloped,above=0.1mm] {} (s2);

\draw (c1)[dashed] -- node[sloped,above=0.1mm] {} (s3);
\draw (c2)[red,dashed] -- node[sloped,above=0.1mm] {} (s3);
\draw (c3)[dashed] -- node[sloped,above=0.1mm] {} (s3);
\draw (c4)[dashed] -- node[sloped,above=0.1mm] {} (s3);
\draw (c1)[dashed] -- node[sloped,above=0.1mm] {} (s4);
\draw (c2)[dashed] -- node[sloped,above=0.1mm] {} (s4);
\draw (c3)[red,dashed] -- node[sloped,above=0.1mm] {} (s4);
\draw (c4)[dashed] -- node[sloped,above=0.1mm] {} (s4);
\end{tikzpicture}
\caption{Constructed bipartite graph $G'$}
\label{fig:constructed_bigraph1}
\end{minipage}
\begin{minipage}[b]{.3\linewidth}
\centering

\definecolor{myblue}{RGB}{80,80,160}
\definecolor{mygreen}{RGB}{80,160,80}

\begin{tikzpicture}[thick,
  fsnode/.style={draw,circle,fill=myblue},
  ssnode/.style={draw,circle,fill=mygreen},
  vsnode/.style={draw,circle,fill=gray}
]

\begin{scope}[start chain=going below,node distance=7mm]
\foreach \i in {1,2,3,4}
  \node[fsnode,on chain] (c\i) [label=left: $c_{\i}$] {};
\end{scope}


\begin{scope}[xshift=2cm,yshift=0cm,start chain=going below,node distance=7mm]
\foreach \i in {3}
  \node[vsnode,on chain] (s\i) [label=right: $s'_{\i}$] {};
\foreach \i in {1,2}
  \node[ssnode,on chain] (s\i) [label=right: $s_{\i}$] {};
\foreach \i in {4}
  \node[vsnode,on chain] (s\i) [label=right: $s'_{\i}$] {};
\end{scope}


\draw[red] (c1) -- node[sloped,above=0.1mm] {} (s1);
\draw (c2) -- node[sloped,above=0.1mm] {} (s1);
\draw (c3) -- node[sloped,above=0.1mm] {} (s1);
\draw (c3) -- node[sloped,above=0.1mm] {} (s2);
\draw[red] (c4) -- node[sloped,above=0.1mm] {} (s2);

\draw (c1)[dashed] -- node[sloped,above=0.1mm] {} (s3);
\draw (c2)[dashed] -- node[sloped,above=0.1mm] {} (s3);
\draw (c3)[red,dashed] -- node[sloped,above=0.1mm] {} (s3);
\draw (c4)[dashed] -- node[sloped,above=0.1mm] {} (s3);
\draw (c1)[dashed] -- node[sloped,above=0.1mm] {} (s4);
\draw (c2)[red,dashed] -- node[sloped,above=0.1mm] {} (s4);
\draw (c3)[dashed] -- node[sloped,above=0.1mm] {} (s4);
\draw (c4)[dashed] -- node[sloped,above=0.1mm] {} (s4);
\end{tikzpicture}
\caption{Constructed bipartite graph $G'$}
\label{fig:constructed_bigraph2}
\end{minipage}
}
\end{figure*}

In this section, we show that how to apply the fully-polynomial almost uniform sampler for perfect matchings \cite{uniform_sample_perfect_matching} $\mathcal{S}(\eta)$ to obtian an almost uniformly sampled matching $M'$  from $\cM$ in bipartite graph $G$.  

Recall that in bipartite graph $G$, $n=|C|$, $\ell =|S|$.
If $n=\ell$, each maximum matching is a perfect matching. Then, we can directly use $\mathcal{S}(\eta)$ to sample a matching almost uniformly.

If $n>\ell$ (note that $n < \ell$ cannot occur due to the assumption of $\cM \neq \varnothing$), we add $n-\ell$ ficticious vertices $\{ s_{\ell+1}, ..., s_{n}\}$ in $S$. In addition, for each ficticious vertex $s_{j}$ ($ \ell+1 \leq j \leq n$), we  add $n$ ficticious edges $(c_1, s_j), ..., (c_n, s_j)$ that connected to each vertex in $C$. 
Let $G'(C,S',E')$ denote this new bipartite graph. There is a one-to-n relationship between the maximum matchings in  $G$  and the perfect matchings in  $G'$.
See Figures~\ref{fig:original_bigraph} to \ref{fig:constructed_bigraph2} for an example. 
Figure~\ref{fig:original_bigraph} illustrates the original bipartite graph $G$ and a valid maximum matching $M=\{e_1,e_5\}$.  Figures~\ref{fig:constructed_bigraph1},\ref{fig:constructed_bigraph2} illustrate the constructed bipartite graph $G'$  and two perfect matchings corresponding to $M$. The gray vertices $s'_3,s'_4$ and dashed edges respectively denote the ficticious vertices and edges, and the red edges denote the perfect matchings.

We first use $\mathcal{S}(\eta)$ to almost uniformly sample a perfect matching $M'_{\textup{perf}}$ in $G'$. Then, we eliminate the ficticious edges in $M'_{\textup{perf}}$ and obtain its corresponding maximum matching $M'$ in original $G$. Because  each maximum matching in original $G$ has the same number of corresponding perfect matchings in  $G'$, the property of the uniform distribution still holds.
Therefore, with $\mathcal{S}(\eta)$, we can  obtian an almost uniformly sampled matching $M'$  from $\cM$ in bipartite graph $G$.

\subsection{Width of Bipartite Graph}
\label{sec:appendix_borda_width}

\defwidth*
Below we show that our width definition (Definition \ref{def:width}) for bipartite graph is equivalent to that in \cite{cpe_icml14}.

First, we recall the definitions of exchange set, exchange class and width in \cite{cpe_icml14} for the problem instance of bipartite graph and maximum matching.

\textbf{Exchange set} $b$ is defined as an ordered pair of disjoint sets $b = (b_+,b_-)$ where $b_+ \cap b_- = \varnothing$ and $b_+,b_- \subseteq E$. Then, we define operator $\oplus$ such that, for any matching $M$ and any exchange set $b = (b_+,b_-)$, we have $M \oplus b := M \setminus b_- \cup b_+$. Similarly, we also define operator such that $M \ominus b := M \setminus b_+ \cup b_-$.

\textbf{Exchange class} $\cB$ for $\cM$ is defined as a collection of exchange sets that satisfies the following property. For any $M_1, M_2 \in \cM$ such that $M_1 \neq M_2$ and for any $e \in M_1 \setminus M_2$, there exists an exchange set $(b_+,b_-) \in \cB $ which satisfies five constraints: (a) $e \in b_-$, (b) $b_+ \subseteq M_2 \setminus M_1$, (c) $b_- \subseteq M_1 \setminus M_2$, 
(d) $M_1 \oplus b \in \cM$  and (e) $M_2 \ominus	b \in \cM$. We use $\textup{Exchange}(\cM)$ to denote the family of all possible exchange classes for $\cM$.

Then, the widths of exchange class $\cB$ and decision class $\cM$ are defined as follows:
$$
\Width(\cB) = \max \limits_{(b_+,b_-)\in \cB} |b_+|+|b_-|,
$$
$$
\Width(\cM) = \min \limits_{\cB \in \textup{Exchange}(\cM)} \Width(\cB).
$$

\begin{figure*}[t]
	\scalebox{1.1}{
		\begin{minipage}[b]{.28\linewidth}
			\centering
			
			\definecolor{myblue}{RGB}{80,80,160}
			\definecolor{mygreen}{RGB}{80,160,80}
			
			\begin{tikzpicture}[thick,
			fsnode/.style={draw,circle,fill=myblue},
			ssnode/.style={draw,circle,fill=mygreen},
			vsnode/.style={draw,circle,fill=gray}
			]
			
			\begin{scope}[start chain=going below,node distance=7mm]
			\foreach \i in {1,2,3,4,5}
			\node[fsnode,on chain] (c\i) [label=left: $c_{\i}$] {};
			\end{scope}
			
			\begin{scope}[xshift=2cm,yshift=-1cm,start chain=going below,node distance=7mm]
			\foreach \i in {1,2,3}
			\node[ssnode,on chain] (s\i) [label=right: $s_{\i}$] {};
			\end{scope}
			
			
			\draw[red] (c1) -- node[sloped,above=0.1mm] {$e_1$} (s1);
			\draw (c2) -- node[sloped,above=0.1mm] {$e_2$} (s1);
			\draw (c2)[red] -- node[sloped,above=0.1mm] {$e_3$} (s2);
			\draw (c3) -- node[sloped,above=0.1mm] {$e_4$} (s2);
			\draw (c4) -- node[sloped,above=0.1mm] {$e_5$} (s3);
			\draw[red] (c5) -- node[sloped,above=0.1mm] {$e_6$} (s3);
			\end{tikzpicture}
			\caption{Maximum matching $M_1$}
			\label{fig:M_1}
		\end{minipage}
		\begin{minipage}[b]{.3\linewidth}
			\centering
			
			\definecolor{myblue}{RGB}{80,80,160}
			\definecolor{mygreen}{RGB}{80,160,80}
			
			\begin{tikzpicture}[thick,
			fsnode/.style={draw,circle,fill=myblue},
			ssnode/.style={draw,circle,fill=mygreen},
			vsnode/.style={draw,circle,fill=gray}
			]
			
			\begin{scope}[start chain=going below,node distance=7mm]
			\foreach \i in {1,2,3,4,5}
			\node[fsnode,on chain] (c\i) [label=left: $c_{\i}$] {};
			\end{scope}
			
			\begin{scope}[xshift=2cm,yshift=-1cm,start chain=going below,node distance=7mm]
			\foreach \i in {1,2,3}
			\node[ssnode,on chain] (s\i) [label=right: $s_{\i}$] {};
			\end{scope}
			
			
			\draw (c1) -- node[sloped,above=0.1mm] {$e_1$} (s1);
			\draw[red] (c2) -- node[sloped,above=0.1mm] {$e_2$} (s1);
			\draw (c2) -- node[sloped,above=0.1mm] {$e_3$} (s2);
			\draw (c3)[red] -- node[sloped,above=0.1mm] {$e_4$} (s2);
			\draw (c4) -- node[sloped,above=0.1mm] {$e_5$} (s3);
			\draw[red] (c5) -- node[sloped,above=0.1mm] {$e_6$} (s3);
			\end{tikzpicture}
			\caption{Maximum matching $M_2$}
			\label{fig:M_2}
		\end{minipage}
		\begin{minipage}[b]{.3\linewidth}
			\centering
			
			\definecolor{myblue}{RGB}{80,80,160}
			\definecolor{mygreen}{RGB}{80,160,80}
			
			\begin{tikzpicture}[thick,
			fsnode/.style={draw,circle,fill=myblue},
			ssnode/.style={draw,circle,fill=mygreen},
			vsnode/.style={draw,circle,fill=gray}
			]
			
			\begin{scope}[start chain=going below,node distance=7mm]
			\foreach \i in {1,2,3,5}
			\node[fsnode,on chain] (c\i) [label=left: $c_{\i}$] {};
			\end{scope}
			
			\begin{scope}[xshift=2cm,yshift=-1cm,start chain=going below,node distance=7mm]
			\foreach \i in {1,2,3}
			\node[ssnode,on chain] (s\i) [label=right: $s_{\i}$] {};
			\end{scope}
			
			
			\draw (c1)[red] -- node[sloped,above=0.1mm] {$e_1$} (s1);
			\draw[red] (c2) -- node[sloped,above=0.1mm] {$e_2$} (s1);
			\draw (c2)[red] -- node[sloped,above=0.1mm] {$e_3$} (s2);
			\draw (c3)[red] -- node[sloped,above=0.1mm] {$e_4$} (s2);
			\draw[red] (c5) -- node[sloped,above=0.1mm] {$e_6$} (s3);
			\end{tikzpicture}
			\caption{Union graph $G(M_1,M_2)$}
			\label{fig:union_graph}
		\end{minipage}
	}
\end{figure*}

We can see that in bipartite graph $G$, for any $M_1, M_2 \in \cM$ such that $M_1 \neq M_2$, their union graph $G(M_1, M_2)$ represents $M_1 \cup M_2$, which can be divided to $(M_1 \setminus M_2) \cup (M_2 \setminus M_1)$ and $M_1 \cap M_2$ (common  edges). Let $\cG$ denote the connected components of $G(M_1, M_2)$. Then, $\cG$ consists of the connected components in $(M_1 \setminus M_2) \cup (M_2 \setminus M_1)$, denoted by $\cG_{\textup{dif}}=\{ G_1(M_1, M_2), G_2(M_1, M_2), \cdots \}$, and those in $M_1 \cap M_2$, denoted by $\cG_{\textup{com}}=\{ e^1, e^2, \cdots\}$. Note that each connected component in $M_1 \cap M_2$ is a single edge. See Figures~\ref{fig:M_1} to \ref{fig:union_graph} for an example. Figures~\ref{fig:M_1},\ref{fig:M_2} illustrate two maximum matchings $M_1,M_2$ in bipartite graph $G$ respectively and Figure~\ref{fig:union_graph} illustrates their union graph $G_(M_1, M_2)$. Then, $G_(M_1, M_2)$ has two connected components, which  respectively fall in $\cG_{\textup{dif}}$ and $\cG_{\textup{com}}$. Specifically,   $\cG_{\textup{dif}}=\{ G_1(M_1, M_2) \}$ where $G_1(M_1, M_2)=\{e_1,e_2,e_3,e_4\}$, and $\cG_{\textup{com}}=\{ e^6\}$.

Then, for any $e \in M_1 \setminus M_2$, there exists some $G_i(M_1, M_2) \in \cG_{\textup{dif}}$ containing $e$. Let $b=G_i(M_1, M_2)$, $b_- = M_1 \cap G_i(M_1, M_2)$ and $b_+ = M_2 \cap G_i(M_1, M_2)$. We can see that $M_1 \oplus b \in \cM$, $M_2 \ominus	b \in \cM$, because the other connected components in $G(M_1, M_2)$ do not change and $M_1 \oplus b $, $M_2 \ominus	b$ are also valid maximum matchings. Thus, $G_i(M_1, M_2)$ is an exchange set for $M_1,M_2$ and $e$ that satisfies the  five constraints (a)-(e). Similarly, any union of multiple connected components in  $\cG_{\textup{dif}}$ containing $G_i(M_1, M_2)$ is an exchange set for $M_1,M_2,e$ that satisfies the  five constraints (a)-(e), and  among these exchange sets, $G_i(M_1, M_2)$ has the smallest size. For the example illustrated in Figures~\ref{fig:M_1} to \ref{fig:union_graph}, $G_1(M_1, M_2)=\{e_1,e_2,e_3,e_4\}$ is the exchange set for $M_1,M_2,e, \  s.t. \  e \in \{e_1,e_2,e_3,e_4\}$, and $\Width(M_1,M_2)=4$. In a similar manner, we can see that for any $M_1, M_2 \in \cM(G), M_1 \neq M_2 $, $\Width(M_1,M_2) \leq 4$. Therefore, $\Width(G)=\max_{ M_1, M_2 \in \cM(G), M_1 \neq M_2 } \Width(M_1,M_2)=4$.

From the above analysis, we can obtain that the exchange class $\cB \in \textup{Exchange}(\cM)$ with minimum $\Width(\cB)$ satisfies that for any $M_1, M_2 \in \cM, M_1 \neq M_2$ and for any  $e \in M_1 \setminus M_2$, $\cB$ only contains the connected component $G_i(M_1, M_2) \  s.t. \  e \in G_i(M_1, M_2)$, not the union  of multiple connected components.
Thus, the minimum $\Width(\cB)$ over  $\cB \in \textup{Exchange}(\cM)$ is exactly the maximum $\Width(M_1,M_2)$ over any $M_1, M_2 \in \cM, M_1 \neq M_2$.
Therefore, for the problem instance of bipartite graph and maximum matching, our definition $\Width(G)=\max_{ M_1, M_2 \in \cM(G), M_1 \neq M_2 } \Width(M_1,M_2)$ is equivalent to that in \cite{cpe_icml14}.

\subsection{Proof of Theorem \ref{thm:borda_ub}}
\label{sec:appendix_borda_ub}

In order to prove Theorem \ref{thm:borda_ub},  we first give a brief introduction of the combinatorial pure exploration setting and the CLUCB algorithm in \cite{cpe_icml14} and extend the original result to that with biased estimates.

In the setting of combinatorial pure exploration, there are $m$ arms and each arm $e \in [m]$ is associated with a reward distribution with mean $w(e)$.
The CLUCB algorithm maintains empirical mean $\bar{w}_t(e)$ and confidence radius $\textup{rad}_t(e)$ for each arm $e \in [m]$ and each timestep $t$. The construction of confidence radius ensures that $|\bar{w}_t(e)-w(e)]| < \textup{rad}_t(e)$ holds with high probability for each arm $e \in [m]$ and each timestep $t$.

In order to prove Theorem \ref{thm:borda_ub}, we first introduce the following lemma as an extended result of the CLUCB algorithm \cite{cpe_icml14} with biased estimates.

\begin{restatable}[CLUCB-bias]{lemma}{clucbbias}\label{lemma:clucb_bias}
In the CLUCB algorithm \cite{cpe_icml14}, 
if $\bar{w}(e)$ is a biased estimator of $w(e)$ and $|\mathbb{E}[\bar{w}(e)] - w(e)| \leq \varepsilon < \frac{\Delta_e}{3\Width(\cM)}$. 
Given any timestep $t > 0$ and suppose that $\forall e \in [m]$, $|\bar{w}(e)-\mathbb{E}[\bar{w}(e)]| < c_t(e)$. For any $e \in [m]$, if $c_t(e) < \frac{\Delta_e}{3\Width(\cM)}-\varepsilon$, then arm $e$ will not be pulled on round $t$.
\end{restatable}

\begin{proof}
We first bound the difference between the  estimator $\bar{w}_t(e)$ and the reward mean ${w}(e)$ as follows:
\begin{align*}
	 |\bar{w}_t(e)- {w}(e)   | 
	\leq & |\bar{w}_t(e)- \mathbb{E}[\bar{w}_t(e)]  |  + | \mathbb{E}[\bar{w}_t(e)] - {w}(e) | 
	\\
	< & c_t(e) + \varepsilon.
\end{align*}
	
Then,  the confidence radius $\textup{rad}_t(e)$  in the Lemma 10 of \cite{cpe_icml14} can be written as $\textup{rad}_t(e)=c_t(e) + \varepsilon$ and we obtain  that given any timestep $t>0$, for any $e \in [m]$, if $c_t(e) < \frac{\Delta_e}{3\Width(\cM)}-\varepsilon$ , then arm $e$ will not be pulled on round $t$.
\end{proof}

\vbox{}

\thmcborda*

\begin{proof}
	First, we prove the correctness of  the $\Cborda$ algorithm (Algorithm  \ref{alg:borda-online}).
		
	Recall that the empirical mean $$\bar{w}_t(e)= \frac{\sum \limits_{s=1}^{T_t(e)} X_s(e) }{T_t(e)},$$
	where $X_s(e)$ denotes the $s$-th observation of the duel between $e$ and  $e'$ that is selected via the almost uniform sampler $\mathcal{S}(\eta)$. Specifically, $X_s(e)$ takes value $1$ if $e$ wins in the $s$-th observation and takes value $0$ otherwise.
	Note that $X_1(e), X_2(e), \ldots, X_t(e)$ are i.i.d.\ random variables.
	
	According to the definition of $\mathcal{S}(\eta)$ (Definition \
	\ref{def:sampler}), in the $s$-th observation of the duel between $e$ and another edge $e'$, $\mathcal{S}(\eta)$ returns a matching $M'$ from distribution $\pi'_s$ that satisfies 
	$$
	d_{tv}(\pi'_s, \pi)= \frac{1}{2} \sum \limits_{M \in \cM} |\pi'_s(x) - \pi(x)| \leq  \eta,
	$$
	where  $\pi$ is the uniform distribution on $\cM$.

	Since $e'$ is the edge at the same position as $e$ in $M'$, we have
	$$\mathbb{E}[X_s(e)]=\sum \limits_{M \in \mathcal{M}} \pi'_s(M) \cdot  p_{e , e(M,j)} ,$$
	where $j$ is the position index of $e$.

	Since  $c_t(e)=\sqrt{\frac{\ln (\frac{4K t^3}{\delta})}{2 T_{t}(e)}}$, according to the Hoeffding's inequality, we have 
	\begin{align*}
	& \Pr \left [ \left |\bar{w}_t(e)- \mathbb{E}[X_1(e)] \right | \geq c_t(e) \right] \\
	& = \Pr \left [ \left | \sum_{s=1}^{T_t(e)} X_s(e) / T_t(e) - \mathbb{E}[X_1(e)] \right | 
		\geq \sqrt{\frac{\ln (\frac{4K t^3}{\delta})}{2 T_{t}(e)}} \right] \\
	& = \sum_{j=1}^t \Pr \left [ \left | \sum_{s=1}^{j} X_s(e) / j - \mathbb{E}[X_1(e)] \right | 
		\geq \sqrt{\frac{\ln (\frac{4K t^3}{\delta})}{2 j}}, T_t(e) = j \right] \\
	& \le  \sum_{j=1}^t \Pr \left [ \left | \sum_{s=1}^{j} X_s(e) / j - \mathbb{E}[X_1(e)] \right | 
		\geq \sqrt{\frac{\ln (\frac{4K t^3}{\delta})}{2 j}} \right] \\
	& \le \sum_{j=1}^t \frac{\delta}{2 K t^3} = \frac{\delta}{2 K t^2}.     
	\end{align*}

	
	In other words, with probability at least $1- \frac{\delta}{2 K t^2}$, we have 
	\begin{align*}
	 |\bar{w}_t(e)- \mathbb{E}[X_1(e)]  | < c_t(e).   
	\end{align*}

	Recall that $w(e)=  \frac{1}{|\mathcal{M}|} \sum_{M\in\cM} p_{e,e(M,j)}$ and $\eta=\frac{1}{8}\varepsilon$.
	Next, we bound the bias between $w(e)$ and $\mathbb{E}[X_1(e)]$.	
	\begin{align*}
	\left| \mathbb{E}[X_1(e)] - w(e) \right|
	= & \left| \sum \limits_{M \in \mathcal{M}} \pi'_1(M) \cdot  p_{e , e(M,j)}  -  \frac{1}{|\mathcal{M}|} \sum_{M\in\cM} p_{e,e(M,j)}  \right|	
	\\
	= & \left| \sum \limits_{M \in \mathcal{M}} \pi'_1(M) \cdot  p_{e , e(M,j)}  -  \sum \limits_{M\in \mathcal{M}} \pi(M)  \cdot  p_{e,e(M,j)}  \right|
	\\
	= & \left| \sum \limits_{M\in \mathcal{M}}   p_{e,e(M,j)} \cdot (\pi'_1(M) - \pi(M) )  \right|
	\\
	\leq  &  \sum \limits_{M\in \mathcal{M}} p_{e,e(M,j)} \cdot \left| \pi'_1(M)   - \pi(M)   \right|
	\\
	\leq  &  \sum \limits_{M\in \mathcal{M}} \left| \pi'_1(M)   - \pi(M)   \right|
	\\
	\leq  &  \frac{1}{4}\varepsilon.
	\end{align*}

	Combining the above reseults, we have that with probability at least $1- \frac{\delta}{2 K t^2}$, 
	
%
%
	

	\begin{align*}
	|\bar{w}_t(e)- {w}(e)   | 
	\leq & |\bar{w}_t(e)- \mathbb{E}[X_1(e)]  |  + | \mathbb{E}[X_1(e)] - {w}(e) | 
	\\
	< & c_t(e)  + \frac{1}{4}\varepsilon.
	\end{align*}
	
	By a union bound over  timestep $t$ and edge $e$, we have that  with probability at least $1- \delta$, for any timestep $t>0$, for any edge $e \in E$, $|\bar{w}_t(e)- {w}(e)| <  c_t(e)  + \frac{1}{4}\varepsilon$.
	
	Thus, with probability at least $1-\delta$, when the $\Cborda$ algorithm terminates, we have
	$$
	{w}(M^B_*) - {w}({\sf Out}) \leq \tilde{w}_t(M^B_*) - \tilde{w}_t({\sf Out}) \leq  \tilde{w}_t(\tilde{M}_t) - \tilde{w}_t({\sf Out}) \leq \ell \varepsilon.
	$$
	Thus, according to Eq. \eqref{eq:reduction},
	$$
	{B}(M^B_*) - {B}({\sf Out}) = \frac{1}{\ell} ( {w}(M^B_*) - {w}({\sf Out}) ) \leq \varepsilon,
	$$
	which completes the proof of the correctness for  the $\Cborda$ algorithm.
	   
	\vbox{}
	   
	 Next, we prove the sample complexity of the $\Cborda$ algorithm (Algorithm  \ref{alg:borda-online}).
	 
	 In the following case (\romannumeral 1) and case (\romannumeral 2), we respectively prove that if $c_t(e)<\frac{\Delta^B_e}{3 \textup{width}(\cM)}- \frac{1}{4}\varepsilon$ or if $c_t(e)<\frac{1}{4}\varepsilon$, edge $e$ will not be pulled as the left arm of duel $(z_t, e')$ in the $\Cborda$ algorithm, \emph{i.e.}, $z_t \neq e$.

	 \paragraph{Case (\romannumeral 1)} If $c_t(e)<\frac{\Delta^B_e}{3 \textup{width}(\cM)}- \frac{1}{4}\varepsilon$, where $\frac{1}{4}\varepsilon < \frac{\Delta^B_e}{3 \textup{width}(\cM)}$, according to  Lemma \ref{lemma:clucb_bias}, we obtain $z_t \neq e$.
	 
	 \paragraph{Case (\romannumeral 2)} If $c_t(e)<\frac{1}{4}\varepsilon$, suppose that edge $e$ is pulled at timestep $t$. Then, 
	 \begin{align*}
	  \tilde{w}_t(\tilde{M}_t) - \tilde{w}_t(M_t)
	  = & \bar{w}_t(\tilde{M}_t) - \bar{w}_t(M_t) + \sum \limits_{ e \in (\tilde{M}_t \setminus M_t) \cup  ( M_t \setminus \tilde{M}_t)} \left( c_t(e) + \frac{1}{4}\varepsilon \right)
	 \\
	  < & \bar{w}_t(\tilde{M}_t) - \bar{w}_t(M_t) + \sum \limits_{ e \in (\tilde{M}_t \setminus M_t) \cup  ( M_t \setminus \tilde{M}_t)} \left( \frac{1}{4}\varepsilon + \frac{1}{4}\varepsilon \right)
	  \\
	  \leq & \bar{w}_t(\tilde{M}_t) - \bar{w}_t(M_t) +  2 \ell \cdot \left( \frac{1}{4}\varepsilon + \frac{1}{4}\varepsilon \right)
	  \\
	  \leq & \ell \varepsilon,
	 \end{align*}
	 which contradicts  the stop condition. 
	 
	 Therefore, we have that if $c_t(e)<\max \{ \frac{\Delta^B_e}{3 \textup{width}(\cM)}- \frac{1}{4}\varepsilon, \frac{1}{4}\varepsilon \}$, then $z_t \neq e$.
	 
	 Since $\frac{1}{8} \cdot \max \{ \frac{\Delta^B_e}{ \Width(G)}, \varepsilon\} <\max \{ \frac{\Delta^B_e}{3 \textup{width}(\cM)}- \frac{1}{4}\varepsilon, \frac{1}{4}\varepsilon \}$, we have that if $c_t(e)<\frac{1}{8} \cdot \max \{ \frac{\Delta^B_e}{ \Width(G)}, \varepsilon\}$, edge $e$ will not be pulled as the left arm of duel $(z_t, e')$ in the $\Cborda$ algorithm, \emph{i.e.}, $z_t \neq e$. 
	 
	 Fix any edge $e \in E$. Let $T(e)$ denote the number of times edge $e$ being pulled as the left arm of duel $(z_t, e')$, \emph{i.e.}, $z_t = e$. Let $t_e$ denote the last timestep when  $z_{t} = e$. It is easy to see that $T_{t_e}(e) = T(e)-1$. According to the above analysis, we see that $c_{t_e}(e) \geq \frac{1}{8} \cdot \max \{ \frac{\Delta^B_e}{ \Width(G)}, \varepsilon\}$. Thus, we have
    \begin{align*}
        c_{t_e}(e) = \sqrt{\frac{\ln (\frac{4K t^3}{\delta} )}{2 (T(e)-1)}} \geq \frac{1}{8} \cdot \max \left\{ \frac{\Delta^B_e}{ \Width(G)}, \varepsilon \right\}
        \\
        T(e) \leq 32 \cdot \min \left\{ \frac{\Width(G)^2}{(\Delta^B_e)^2}, \frac{1}{\varepsilon^2} \right\} \cdot \ln \left(\frac{4K T^3}{\delta} \right)+1
    \end{align*}
Recall that $H^B_{\varepsilon}:= \sum_
{e \in E} \min \{ \frac{\Width(G)^2}{(\Delta^B_e)^2}, \frac{1}{\varepsilon^2}\}$. Taking summation over $e \in E$, we have
\begin{align}
T \leq 32 H^B_{\varepsilon} \ln \left(\frac{4K t^3}{\delta} \right)+m.    \label{eq:T_leq}
\end{align}

Below we prove that 
\begin{align}
 T \leq  985 H^B_{\varepsilon} \ln \left(\frac{4 H^B_{\varepsilon}}{\delta} \right)+2m.   \label{eq:borda_sample_bound}
\end{align}

If $m \geq \frac{1}{2} T$, then Eq. \eqref{eq:borda_sample_bound} holds immediately.
Next, we consider the case when $m < \frac{1}{2} T$.  Since $T > m$, we can write
$$
T=C H^B_{\varepsilon} \ln \left(\frac{4 H^B_{\varepsilon}}{\delta} \right)+m,
$$
where $C$ is some positive constant.

If $C \leq 985$, then we see that Eq. \eqref{eq:borda_sample_bound} holds.
On the contrary, if $C > 985$, from Eq. \eqref{eq:T_leq}, we have
\begin{align*}
T \leq & m + 32 H^B_{\varepsilon} \ln \left( \frac{4K T^3}{\delta} \right)
\\
= & m + 32 H^B_{\varepsilon} \ln \left( \frac{4K }{\delta} \right) + 96 H^B_{\varepsilon} \ln \left(C H^B_{\varepsilon} \ln \left( \frac{4 H^B_{\varepsilon}}{\delta} \right)+m \right)
\\
\leq & m + 32 H^B_{\varepsilon} \ln \left( \frac{4K }{\delta} \right) + 96 H^B_{\varepsilon} \ln \left( 2C H^B_{\varepsilon} \ln \left( \frac{4 H^B_{\varepsilon}}{\delta} \right) \right)
\\
= & m + 64 H^B_{\varepsilon} \ln \left( \frac{4K }{\delta} \right) + 96 H^B_{\varepsilon} \ln (2C) + 96 H^B_{\varepsilon} \ln (  H^B_{\varepsilon} )+ 96 H^B_{\varepsilon} \ln \left(  \ln \left( \frac{4 H^B_{\varepsilon}}{\delta} \right) \right)
\\
\leq & m + 64 H^B_{\varepsilon} \ln \left( \frac{4 H^B_{\varepsilon}}{\delta} \right) + 96  \ln (2C) H^B_{\varepsilon} \ln \left( \frac{4 H^B_{\varepsilon}}{\delta} \right) + 96 H^B_{\varepsilon} \ln \left( \frac{4 H^B_{\varepsilon}}{\delta} \right) + 96 H^B_{\varepsilon} \ln \left ( \frac{4 H^B_{\varepsilon}}{\delta} \right)
\\
= & m + (256+ 96  \ln (2C)) H^B_{\varepsilon} \ln \left( \frac{4 H^B_{\varepsilon}}{\delta} \right) 
\\
< & m + C H^B_{\varepsilon} \ln \left( \frac{4 H^B_{\varepsilon}}{\delta} \right)
\\
= & T,
\end{align*}
which makes a contradiction. Therefore, we have $C \leq 985$ and complete the proof of Eq. \eqref{eq:borda_sample_bound}. Theorem \ref{thm:borda_ub} follows immediately from Eq. \eqref{eq:borda_sample_bound}.

\end{proof}

\subsection{Exact Algorithm for Identifying Borda Winner}
\label{sec:appendix_borda_exact}
In Algorithm \ref{alg:borda_exact}, we present the detailed algorithm $\CbordaExact$ for identifying the exact Borda winner. Then, in the following we give the detailed proof of its sample complexity upper bound (Theorem \ref{thm:borda_ub_exact}).

\begin{algorithm}[th]
\caption{$\CbordaExact$}
\label{alg:borda_exact} \setstretch{1.1}
\begin{algorithmic}[1]
    \STATE {\bfseries Input:} confidence  $\delta$, bipartite graph $G$, decision class $\mathcal{M}$, maximization oracle $\Oracle(\cdot)$: $\mathbb{R}^m \rightarrow \mathcal{M}$ and  almost uniform sampler  for perfect matchings $\mathcal{S}(\eta)$
    \FOR{$q = 1,2,\dots$}
    \STATE $\varepsilon_q \leftarrow \frac{1}{2^q}$
    \STATE $\delta_q \leftarrow \frac{\delta}{2q^2}$
	\STATE Set bias parameter $\eta_q \leftarrow \frac{1}{8}\varepsilon_q$
	\STATE Initialize $T_{1}(e) \leftarrow 0$ and $\bar{w}_{1}(e) \leftarrow 0$ for all $e \in E$ 
	\FOR  {$t=1,2,...$}
		\STATE $M_t \leftarrow \Oracle (\bar{\boldsymbol{w}}_t)$ 
		\STATE Compute confidence radius $c_t(e) \leftarrow \sqrt{\frac{\ln (\frac{4K t^3}{\delta_q})}{2 T_{t}(e)}}$ for all $e \in E$ \quad //  $\frac{x}{0}:=1$ for any $x$
		\FOR{ \textup{all } $e \in E$}
			\IF { $e \in M_t$ } 
				\STATE $ \tilde{w}_t(e) \leftarrow \bar{w}_t(e) - c_t(e) -\frac{1}{4}\varepsilon_q $ 
			\ELSE
				\STATE $ \tilde{w}_t(e) \leftarrow \bar{w}_t(e) + c_t(e) + \frac{1}{4}\varepsilon_q $ 
			\ENDIF \quad 	// $\bar{w}_t(e):=0$ if $T_{t}(e)=0$
		\ENDFOR
		\STATE $\tilde{M}_t \leftarrow \Oracle (\tilde{\boldsymbol{w}}_t)$ 
		\IF{ $\tilde{w}_t(\tilde{M}_t) = \tilde{w}_t(M_t)$ }
			\STATE $\Out \leftarrow M_t$ 
			\STATE  {\bfseries return } $\Out$ 
		\ENDIF
		\IF{$\tilde{w}_t(\tilde{M}_t) - \tilde{w}_t(M_t) \leq \ell \varepsilon_q$}
		    \STATE {\bfseries break }
		\ENDIF
		\STATE $z_t \leftarrow \mathop{\arg \max}_{ e \in (\tilde{M}_t \setminus M_t) \cup  ( M_t \setminus \tilde{M}_t)} c_t(e)$ 
		\STATE Sample a matching $M'$ from $\mathcal{M}$ using $\mathcal{S}(\eta_q)$ 
		\STATE Pull the duel $(z_t, e')$, where $e'=e(M', s(z_t))$
		\STATE Update empirical means $\bar{w}_t(z_t)$ according to the winning or lossing of $z_t$ and set $T_{t+1}(z_t) \leftarrow T_{t}(z_t) +1$
    \ENDFOR
    \ENDFOR
\end{algorithmic}
\end{algorithm}

\thmcbordaexact*

\begin{proof}
    First, we prove the correctness of the $\CbordaExact$ algorithm (Algorithm \ref{alg:borda_exact}).
    
    Note that in epoch $q$, the $\CbordaExact$ algorithm  performs a subroutine of the $\Cborda$ algorithm (Algorithm \ref{alg:borda-online}) with confidence $\delta_q$ and accuracy $\varepsilon_q$.
    Then, using similar analysis in the proof of the correctness (Theorem \ref{thm:borda_ub}) of the $\Cborda$ algorithm, we have that for any epoch $q$, with probability at least $1- \delta_q$, for any edge $e \in E$, $|\bar{w}_t(e)- {w}(e)| <  c_t(e)  + \frac{1}{4}\varepsilon_q$.
    
    Since $\sum_{q=1}^{\infty} \delta_q = \sum_{q=1}^{\infty} \frac{\delta}{2q^2} \leq \delta$, by a union bound over $q$, we have that with probability at least $1- \delta$, for any epoch $q$, for any edge $e \in E$, $|\bar{w}_t(e)- {w}(e)| <  c_t(e)  + \frac{1}{4}\varepsilon_q$.

	Thus, with probability at least $1-\delta$, when the $\CbordaExact$ algorithm terminates, \emph{i.e.}, $\tilde{w}_t(\tilde{M}_t) = \tilde{w}_t(M_t)$, we have that for any $M \neq M_t$,
	$$
	    \tilde{w}_t(M_t) \geq \tilde{w}_t(M)
	$$
	$$
	     \sum_{e \in M_t \setminus M} \left ( \bar{w}_t(e) - c_t(e) -\frac{1}{4}\varepsilon_q \right) \geq  \sum_{e \in M \setminus M_t} \left( \bar{w}_t(e) + c_t(e) +\frac{1}{4}\varepsilon_q \right)
	$$
	$$
	    \sum_{e \in M_t \setminus M}  {w}(e)  >  \sum_{e \in M \setminus M_t}  {w}(e) 
	$$
	$$
	     {w}(M_t)  >    {w}(M)
	$$
	Therefore, we obtain $\Out=M_t=M^B_*$ and complete the proof of the correctness for  the $\CbordaExact$ algorithm.

    Next, we prove the sample complexity of the $\CbordaExact$ algorithm (Algorithm \ref{alg:borda_exact}).
    Using similar analysis in the proof of the sample complexity (Theorem \ref{thm:borda_ub}) of the $\Cborda$ algorithm, 
    we have that 
    with probability at least $1- \delta_q$, the number of samples in epoch $q$ is bounded by
    \[T_q \leq
    O \left( \sum_
    {e \in E} \min \left\{ \frac{\Width(G)^2}{(\Delta^B_e)^2}, \frac{1}{\varepsilon^2} \right\} \ln \left ( \frac{1}{\delta_q} \cdot  \sum_
    {e \in E} \min \left\{ \frac{\Width(G)^2}{(\Delta^B_e)^2}, \frac{1}{\varepsilon^2} \right\} \right)  \right).
    \]
    
Let $q^* = \left \lfloor \log_2 (\frac{\ell}{\Delta^B_{\textup{min}}})  \right \rfloor +1$ denote the first epoch that satisfies $\varepsilon_q^* <  \frac{\Delta^B_{\textup{min}}}{\ell}$.
In the following, we show that in epoch $q^*$, the $\CbordaExact$ algorithm terminates \emph{i.e.}, $\tilde{w}_t(\tilde{M}_t) = \tilde{w}_t(M_t)$ holds before $\tilde{w}_t(\tilde{M}_t) - \tilde{w}_t(M_t) \leq \ell \varepsilon_q$. 

Suppose that, in epoch $q^*$,  $\tilde{w}_t(\tilde{M}_t) - \tilde{w}_t(M_t) \leq \ell \varepsilon_q^*$ holds before $\tilde{w}_t(\tilde{M}_t) = \tilde{w}_t(M_t)$, which implies that the $\CbordaExact$ algorithm enters epoch $q^*+1$. Then, at the last timestep of epoch $q^*$,
$
\tilde{w}_t(\tilde{M}_t) - \tilde{w}_t(M_t) \leq \ell \varepsilon_q^* < \Delta^B_{\textup{min}}
$
and 
$\tilde{w}_t(\tilde{M}_t) \neq \tilde{w}_t(M_t)$.

Since $\tilde{w}_t(\tilde{M}_t) \neq \tilde{w}_t(M_t)$, $\tilde{M}_t \neq M_t$. Thus, with probability at least $1-\delta$, we have
$$
\sum_{e \in \tilde{M}_t \setminus M_t} \left( \bar{w}_t(e) + c_t(e) +\frac{1}{4}\varepsilon_q \right) - \sum_{e \in M_t \setminus \tilde{M}_t} \left ( \bar{w}_t(e) - c_t(e) -\frac{1}{4}\varepsilon_q \right) < \Delta^B_{\textup{min}} 
$$
$$
\sum_{e \in \tilde{M}_t \setminus M_t}  {w}(e)  - \sum_{e \in M_t \setminus \tilde{M}_t}  {w}(e) < \Delta^B_{\textup{min}} 
$$
$$
  {w}(\tilde{M}_t)  -   {w}(M_t) < \Delta^B_{\textup{min}} ,
$$
which contradicts the definition of $\Delta^B_{\textup{min}}$. 

Therefore, in epoch $q^*$,  the $\CbordaExact$ algorithm terminates. Note that if the $\CbordaExact$ algorithm terminates before epoch $q^*$, our proof of sample complexity still holds.

Now we bound the total number of samples from epoch $1$ to $q^*$ as
    
\begin{align*}
T \leq & \sum_{q=1}^{q^*} T_q
\\
 = & 
    O \left( \sum_{q=1}^{q^*} \sum_
    {e \in E} \min \left\{ \frac{\Width(G)^2}{(\Delta^B_e)^2}, \frac{1}{\varepsilon^2} \right\} \ln \left ( \frac{1}{\delta_q} \cdot  \sum_
    {e \in E} \min \left\{ \frac{\Width(G)^2}{(\Delta^B_e)^2}, \frac{1}{\varepsilon^2} \right\}   \right) \right)
\\
 = & 
    O \left( \sum_{q=1}^{q^*} \sum_
    {e \in E}  \frac{\Width(G)^2}{(\Delta^B_e)^2} \ln \left ( \frac{1}{\delta_q} \cdot  \sum_
    {e \in E}  \frac{\Width(G)^2}{(\Delta^B_e)^2}   \right) \right)
    \\
 = & 
    O \left( \sum_{q=1}^{q^*} \Width(G)^2 H^B \ln \left ( \frac{2 q^2}{\delta} \cdot  \Width(G)^2 H^B \right) \right)
   \\
 = & 
    O \left( \sum_{q=1}^{q^*} \Width(G)^2 H^B  \left(\ln \left ( \frac{\Width(G) H^B}{\delta}   \right) + \ln q \right) \right)
       \\
 = & 
    O \left( q^* \Width(G)^2 H^B  \left(\ln \left ( \frac{\Width(G) H^B}{\delta}   \right) + \ln q^* \right) \right)
           \\
 = & 
    O \left( \ln \left(\frac{\ell}{\Delta^B_{\textup{min}}} \right) \Width(G)^2 H^B  \left(\ln \left ( \frac{\Width(G) H^B}{\delta}   \right) + \ln \ln \left(\frac{\ell}{\Delta^B_{\textup{min}}} \right) \right) \right),
\end{align*}
which completes the proof of Theorem \ref{thm:borda_ub_exact}.

\end{proof}

\subsection{Proof of Theorem \ref{thm:borda_lb}}
\label{sec:appendix_borda_lb}

\thmbordalb*

\begin{proof}
    For ease of notation, we first introduce the following definition.
    
    \begin{restatable}[Next-to-optimal]{definition}{}\label{def:hardness-borda}
    For any edge $e \in E$, we define the next-to-optimal set associated with $e$ as follows:
    $$
    M^B_e=\left\{\begin{matrix}
    \argmax \limits_{M \in \mathcal{M}: e \in M} w(M)  \text{ \quad if } e \notin M^B_*,
    \\ 
    \argmax \limits_{M \in \mathcal{M}: e \notin M} w(M)  \text{ \quad if } e \in M^B_*.
    \end{matrix} \right.
    $$
    \end{restatable}
    
    Note that, according to the  definition of $\Delta^B_e$ (Definition \ref{def:gap-borda}), we have  $w(M^B_*)-w(M^B_e)=\Delta^B_e$.

    Fix an instance $\mathcal{I}$ of combinatorial pure exploration for Borda dueling bandits and a $\delta$-correct algorithm $\mathbb{A}$. 
	In instance $\mathcal{I}$, $M_*^B$ is the Borda winner and $M_x$ is a suboptimal super arm.
	Let $T_{e_z, e_k}$ be the expected number of samples drawn from the duel $(e_z, e_k)$ when $\mathbb{A}$ runs on instance $\mathcal{I}$.
	
	We consider the following alternative instance $\mathcal{I'}$. For an edge $e=(c_i, s_j)$, we change all the distributions of duels $(e, \tilde{e})$ $s.t.$ $\tilde{e} \in E_j \setminus \{ e \} $ as follows:
	
	$$ 
	p'_{e, \tilde{e}}=\left \{
	\begin{aligned}
	p_{e, \tilde{e}}+ \frac{|\mathcal{M}|}{|\mathcal{M}| - |\cM_e|} \cdot \Delta^B_e  \qquad  \textup{if } e \notin M^B_*,
	\\ 
	p_{e, \tilde{e}}- \frac{|\mathcal{M}|}{|\mathcal{M}| - |\cM_e|} \cdot \Delta^B_e  \qquad  \textup{if } e \in M^B_*.
	\end{aligned}
	\right. 
	$$

	Then, for the next-to-optimal matching $M^B_e$, 
	\begin{align*}
	  w'(M^B_e)-w'(M^B_*) 
	 \geq & w(M^B_e)-w(M^B_*) +  \frac{1}{|\cM|} \sum \limits_{M \in \cM \setminus \cM_e} \left| p'_{e, e(M,j)}-p_{e, e(M,j)} \right|
	 \\
	= & w(M^B_e)-w(M^B_*) +  \frac{1}{|\cM|} \sum \limits_{M \in \cM \setminus \cM_e} \frac{|\mathcal{M}|}{|\mathcal{M}| - |\cM_e|} \cdot \Delta^B_e 
	\\
	= &w(M^B_e)-w(M^B_*) + \Delta^B_e 
	\\
	= & 0.   
	\end{align*}
	
    Thus, we can see that in instance $\mathcal{I'}$, $M^B_e$ is the Borda winner instead.
	
	Using Lemma 1 in \cite{kaufmann2016complexity}, fixing $e=(c_i,s_j)$, we can obtain 
	$$ \sum \limits_{\tilde{e} \in E_j \setminus \{ e \} } T_{e, \tilde{e}} \cdot d(p_{e, \tilde{e}}, p'_{e, \tilde{e}}) \geq d(1-\delta, \delta). $$
	
	For $\delta \in (0,0.1)$, we have $d(1-\delta,  \delta) \geq 0.4 \ln (\frac{1}{\delta} )$.
	Suppose that, for some constant $ \gamma \in (0, \frac{1}{4})$,  $\frac{1}{2}-\gamma \leq p_{e_i,e_j} \leq \frac{1}{2}+\gamma , \  \forall e_i,e_j \in E$ and $\frac{|\mathcal{M}|}{|\mathcal{M}| - |\cM_e|} \leq \frac{1-4\gamma}{4 \gamma l}, \  \forall e \in E$.
	Then, for any $e \in E$, $ \Delta^B_e \leq 2 \gamma \ell$. For any $e_i,e_j \in E$  ($e_i \neq e_j$), $ \gamma \leq p'_{e_i,e_j} \leq 1-\gamma$ and $d(p_{e_i,e_j}, p'_{e_i,e_j}) \leq \frac{ ( p_{e_i,e_j}-p'_{e_i,e_j} )^2 }{p'_{e_i,e_j} ( 1-p'_{e_i,e_j} )} \leq \frac{1}{\gamma ( 1-\gamma )}  ( p_{e_i,e_j}-p'_{e_i,e_j} )^2 $.
	
	Therefore, fixing $e=(c_i,s_j)$, we have
	$$
	 \frac{1}{\gamma ( 1-\gamma )} \sum \limits_{\tilde{e} \in E_j \setminus \{ e \} } T_{e, \tilde{e}} \cdot (p_{e, \tilde{e}} - p'_{e, \tilde{e}})^2  \geq 0.4 \ln \Big(\frac{1}{\delta} \Big)
	$$
	$$
	\frac{1}{\gamma ( 1-\gamma )} \left ( \frac{|\mathcal{M}|}{|\mathcal{M}| - |\cM_e|} \cdot \Delta^B_e   \right )^2 \sum \limits_{\tilde{e} \in E_j \setminus \{ e \} } T_{e, \tilde{e}}   \geq 0.4 \ln \Big(\frac{1}{\delta} \Big)
	$$
	$$
	 \sum \limits_{\tilde{e} \in E_j \setminus \{ e \} } T_{e, \tilde{e}}    \geq 0.4 \gamma ( 1-\gamma ) \left (\frac{4 \gamma \ell}{1-4\gamma} \right)^2 \frac{1}{ (\Delta^B_e)^2 }   \ln \Big(\frac{1}{\delta} \Big)
	$$
	
	 We can perform the similar distribution changes on any edge $e \in E$. Therefore, we can obtain
	 \begin{align*}
	  \sum \limits_{e_z< e_k} T_{e_z, e_k}
	 = & \sum \limits_{j=1}^{\ell} \sum \limits_{ \begin{subarray}{c}  e_z,e_k \in E_j\\  e_z<e_k \end{subarray} } T_{e_z, e_k} 
	 \\
	 = & \frac{1}{2} \sum \limits_{j=1}^{\ell} \sum \limits_{   e_z \in E_j  } \sum \limits_{ e_k \in E_j \setminus \{e_z\} } T_{e_z, e_k} 
	 \\
	 = & \frac{1}{2} \sum \limits_{   e  \in E   } \sum \limits_{ \tilde{e} \in E_{s(e)} \setminus \{e\}} T_{e, \tilde{e}} 
	 \\
	 \geq & 0.2  \gamma ( 1-\gamma ) \left (\frac{4 \gamma \ell}{1-4\gamma} \right)^2 \ln \Big(\frac{1}{\delta} \Big) \sum \limits_{   e \in E   }  \frac{1}{ (\Delta^B_e)^2 }   
	 \\
	 = & \Omega \left ( H^B   \ln \Big(\frac{1}{\delta} \Big) \right ),
	 \end{align*}
	 which completes the proof of Theorem \ref{thm:borda_lb}.

\end{proof}

\section{Omitted Proofs in Section \ref{sec:condorcet}}
In this section, we will introduce the efficient pure exploration algorithm $\Carcond$ to find a Condorcet winner. We will first introduce the efficient pure exploration part assuming there exist ``an oracle'' that performs like a black-box, and we will show the correctness and the sample complexity of $\Carcond$ given the oracle. Next, we will present the details of the oracle, and show that the time complexity of the oracle is polynomial. Then, we will apply the verification framework to further improve our sample complexity. Finally, we will give the sample complexity lower bound for finding the Condorcet winner.

\subsection{Accept-reject algorithm for combinatorial pure exploration}
In this section, we prove Theorem \ref{thm:carcond}. The proof is divided into 2 parts: the first part shows the correctness of $\Carcond$, and the second part bounds the sample complexity. We begin with the first part.

\paragraph{Correctness of $\Carcond$}
\begin{definition}[Sampling is nice]\label{def:sampling-nice}
    Define event $\cN_t:=\{  \underline{p}_t(e_1,e_2) \le p_{e_1,e_2} \le \bar{p}_t(e_1,e_2),\forall e_1 \neq e_2, e_1,e_2\in E_j \}$. Furthermore, we use $\cN = \cap_{t\ge 1}\cN_t$ to denote the case when $\cN_t$ happens for at every round $t$.
\end{definition}

We have the following lemma to show that $\cN$ is a high probability event.

\begin{lemma}\label{lem:sampling-nice-high-prob}
    $\cN$ is a high probability event. Formally, we have
    \[\Pr\{\lnot \cN\} \le \delta.\]
\end{lemma}

\begin{proof}
    The proof is an application of the Hoeffding Inequality and the union bound. We first bound $\lnot\cN_{t}$, and we have
    \begin{align*}
        \Pr\{\lnot\cN_t\} =& \Pr\{\exists e_1,e_j, |\hat p_t(e_i,e_j) - p_{e_i,e_j}| > c_t(e_i,e_j)\} \\
        \le& \sum_{{Comparable\ e_i,e_j}}\Pr\{|\hat p_t(e_i,e_j) - p_{e,e_j}| > c_t(e_i,e_j)\} \\
        \le& \sum_{{Comparable\ e_i,e_j}}\Pr\left\{\big|\hat p_t(e_i,e_j) - p_{e_i,e_j}\big| > \sqrt{\frac{\ln(4Kt^3/\delta)}{2T_t(e_i,e_j)}}\right\} \\
        \le& \sum_{{Comparable\ e_i,e_j}}\sum_{k=1}^t\Pr\left\{\big|\hat p_t(e_i,e_j) - p_{e_i,e_j}\big| > \sqrt{\frac{\ln(4Kt^3/\delta)}{2T_t(e_i,e_j)}},T_t(e_i,e_j) = k\right\} \\
        \le& \sum_{{Comparable\ e_i,e_j}}\sum_{k=1}^t \exp\left(-2k \left(\sqrt{\frac{\ln(4Kt^3/\delta)}{2T_t(e_i,e_j)}}\right)^2\right) \\
        \le& \sum_{{Comparable\ e_i,e_j}}\sum_{k=1}^t \frac{\delta}{4Kt^3} \\
        \le& \frac{\delta}{2t^2}.
    \end{align*}
    Then we have
    \begin{align*}
        \Pr\{\lnot\cN\} =& \Pr\{\exists t \ge 1, \lnot\cN_t\} \\
        \le& \sum_{t\ge 1}\Pr\{\lnot\cN_t\} \\
        \le& \sum_{t\ge 1}\frac{\delta}{2t^2} \\
        \le& \delta.
    \end{align*}
\end{proof}

Then we have the key lemma for the correctness of $\Carcond$. The lemma says that, when $\cN$ happens, $\Carcond$ will not wrongly classify the edges.
\begin{lemma}\label{lem:key-correctness}
    Suppose the optimal super arm is denoted by $M_*^C$. If $\cN$ happens, then at the end of every round $t$, we have
    \[A_t \subseteq M_*^C, R_t\subseteq (M_*^C)^c.\]
\end{lemma}

\begin{proof}
    We use induction to prove that $A_t \subseteq M_*^C, R_t\subseteq (M_*^C)^c$ at the end of every round $t$ if $\cN$ happens.
    
    Note that the optimal matching can be solved by the following minimax optimization problem
    \[\max_{x\in\chi_{\cM}}\min_{y\in\chi_{\cM}}\frac{1}{\ell}x^TPy.\]
    It is known that when $x = \chi_{M_*^C}$, the minimax optimization problem will reach its optimal value $\frac{1}{2}$. Because our assumption, we have for any $y\in \chi_{\cM}, y \neq \chi_{M_*^C}$,
    \[\frac{1}{\ell}\chi_{M_*^C}^T P y \ge \frac{1}{2} + \Delta^{Cond},\frac{1}{\ell}y^T P \chi_{M_*^C} \le \frac{1}{2} - \Delta^{Cond}.\]
    Suppose that at time $t-1$, the induction is correct, i.e. $A_{t-1} \subseteq M_*^C, R_{t-1}\subseteq (M_*^C)^c$. We use $\cP( \cM, A, R)$ to denote the arm distributions that is a linear combination of the matchings in $\cM$ such that $A$ must appear in the super arm and $R$ must not appear in the super arm. Then for any set $\cP,\cQ$ such that $|x|_1 = |y|_1 = 1,\forall x\in\cP,y\in\cQ$, we have
    \[\max_{x\in\cP}\min_{y\in\cQ}\frac{1}{\ell}x^T\underline{P}y \le \max_{x\in\cP}\min_{y\in\cQ}\frac{1}{\ell}x^T P y \le \max_{x\in\cP}\min_{y\in\cQ}\frac{1}{\ell}x^T\bar P y.\]
    Suppose that time $t$ belongs to epoch $q$. If $e\in M_*^C$, then we have
    \begin{align*}
        \text{ExL} \le& \max_{x\in\cP( \cM, A_{t-1}, R_{t-1}\cup\{e\})}\min_{y\in\cP( \cM, A_{t-1}, R_{t-1})}\frac{1}{\ell}x^T\underline{P}y \\
        \le& \max_{x\in\cP( \cM, A_{t-1}, R_{t-1}\cup\{e\})}\frac{1}{\ell}x^T\underline{P}\chi_{M_*^C} \\
        \le& \max_{x\in\cP( \cM, A_{t-1}, R_{t-1}\cup\{e\})}\frac{1}{\ell}x^T P\chi_{M_*^C} \\
        \le& \frac{1}{2}.
    \end{align*}
    We also have
    \begin{align*}
        \text{InU} + \varepsilon_q \ge& \max_{x\in\cP( \cM, A_{t-1}\cup\{e\}, R_{t-1})}\min_{y\in\cP( \cM, A_{t-1}, R_{t-1})}\frac{1}{\ell}x^T\bar P y \\
        \ge& \min_{y\in\cP( \cM, A_{t-1}, R_{t-1})}\frac{1}{\ell}\chi_{M_*^C}^T\bar P y \\
        \ge& \min_{y\in\cP( \cM, A_{t-1}, R_{t-1})}\frac{1}{\ell}\chi_{M_*^C}^T P y \\
        \ge& \frac{1}{2}.
    \end{align*}
    Then we know that $\text{InU} + \varepsilon \ge \text{ExL}$ and the algorithm will not put $e$ into the set $R_t$. On the other hand, if $e\notin M_*^C$, then
    \begin{align*}
        \text{InL} \le& \max_{x\in\cP( \cM, A_{t-1}\cup\{e\}, R_{t-1})}\min_{y\in\cP( \cM, A_{t-1}, R_{t-1})}\frac{1}{\ell}x^T\underline{P}y \\
        \le& \max_{x\in\cP(\cM, A_{t-1}\cup\{e\}, R_{t-1})}\frac{1}{\ell}x^T\underline{P}\chi_{M_*^C} \\
        \le& \max_{x\in\cP(\cM, A_{t-1}\cup\{e\}, R_{t-1})}\frac{1}{\ell}x^T P\chi_{M_*^C} \\
        \le& \frac{1}{2},
    \end{align*}
    and
    \begin{align*}
        \text{ExU} + \varepsilon_q \ge& \max_{x\in\cP(\cM, A_{t-1}, R_{t-1}\cup\{e\})}\min_{y\in\cP(\cM, A_{t-1}, R_{t-1})}\frac{1}{\ell}x^T\bar P y \\
        \ge& \min_{y\in\cP(\cM, A_{t-1}, R_{t-1})}\frac{1}{\ell}\chi_{M_*^C}^T\bar P y \\
        \ge& \min_{y\in\cP(\cM, A_{t-1}, R_{t-1})}\frac{1}{\ell}\chi_{M_*^C}^T P y \\
        \ge& \frac{1}{2}.
    \end{align*}
    Thus, we know that $\text{ExU} + \varepsilon \ge \text{InL}$ and the algorithm will not put $e$ into the set $A_t$.
\end{proof}

With the help of Lemma \ref{lem:key-correctness}, we have the following lemma summarize the correctness of $\Carcond$.

\begin{lemma}[Correctness of $\Carcond$]\label{lem:carcond-correct}
    When $\cN$ happens, if $\Carcond$ stops, then $\Carcond$ will return the Condorcet winner $M_*^C$.
\end{lemma}

\begin{proof}
    When $\Carcond$ stops at round $t$, it means that $|A_t| = \ell$. Then from the previous lemma (Lemma \ref{lem:key-correctness}), we know that when $\cN$ happens, $A_t \subseteq M^C_*$. However, $M^C_* = \ell$ and thus $A_t = M^C_*$. In this way, $\Carcond$ returns the correct (unique) Condorcet winner.
\end{proof}

\paragraph{Sample complexity of $\Carcond$}
In the previous part, we show that if the algorithm stops, then with high probability, the output is correct. Now in this part, we show that with high probability, the algorithm with stop, and formally, we bound the sample complexity of $\Carcond$. First, we recall the definition of Gap in the Condorcet winner case.

\defgapcarcond*

\begin{lemma}[Sample Complexity of $\Carcond$]\label{lem:carcond-sample}
    If $\cN$ happens, the sample complexity of $\Carcond$ is bounded by
    \[O\left(\sum_{j=1}^{\ell}\sum_{e_1\neq e_2,e_1,e_2\in E_j}\frac{1}{(\Delta^C_{e_1,e_2})^2}\ln\left(\frac{K}{\delta(\Delta^C_{e_1,e_2})^2}\right)\right).\]
\end{lemma}

\begin{proof}
    We first prove that, at round $t$ in epoch $q$ such that $e\in U_t$ and $c_t < \frac{\Delta_{e}}{6}$ for $\Delta_e > 6\varepsilon_q$, the algorithm $\Carcond$ will classify arm $e$ into either $A_{t+1}$ or $R_{t+1}$. For simplicity, we denote $c_t := \sqrt{\frac{\ln(4Kt^3/\delta)}{2t}}$, and it is the confidence radius for those arms in set $U_{t}$ after the exploration in round $t$.
    
    \paragraph{Case 1.} If arm $e\in M_*^C,\Delta_e > 6\varepsilon_q, c_t < \frac{\Delta_{e}}{6}$, and $e\notin A_{t-1}$, we show that $e\in A_{t}$. Note that if $\cN$ happens, we have
    \begin{align*}
        \text{InL} + \varepsilon_q \ge& \max_{x\in\cP(\cM, A_{t-1}\cup\{e\}, R_{t-1})}\min_{y\in\cP(\cM, A_{t-1}, R_{t-1})}\frac{1}{\ell}x^T\underline{P}y \\
        \ge& \min_{y\in\cP(\cM, A_{t-1}, R_{t-1})}\frac{1}{\ell}\chi_{M_*^C}^T\underline{P}y \\
        \ge& \min_{y\in\cP(\cM, A_{t-1}, R_{t-1})}\frac{1}{\ell}\chi_{M_*^C}^TPy - \frac{1}{\ell}\ell 2c_t \\
        \ge& \frac{1}{2} - 2 c_t,
    \end{align*}
    and
    \begin{align*}
        \text{ExU} \le & \max_{x\in\cP(\cM, A_{t-1}, R_{t-1}\cup\{e\})}\min_{y\in\cP(\cM, A_{t-1}, R_{t-1})}\frac{1}{\ell}x^T\bar P y \\
        \le& \max_{x\in\cP(\cM, A_{t-1}, R_{t-1}\cup\{e\})}\frac{1}{\ell}x^T\bar P \chi_{M_*^C} \\
        \le& \max_{x\in\cP(\cM, A_{t-1}, R_{t-1}\cup\{e\})}\frac{1}{\ell}x^T\bar P \chi_{M_*^C} + \frac{1}{\ell}\ell 2c_t \\
        \le& \frac{1}{2} - \Delta_e + 2 c_t.
    \end{align*}
    The reasons between the inequality between line 2 and line 3 are: 1. The matrix $P_t,\underline{P}_{t-1}$ are all diagonal block matrices and they can be partitioned into $\ell$ small nonzero matrices; 2. Although for the edge $e'\in A_{t-1}\cup R_{t-1}$, the confidence radius is larger than $c_t$, however, in the computation we will never use that larger confidence radius. If $e'\in A_{t-1}$, then at the same position in the matching, $y$ also chooses $e'$ and we know the exact value $P_{e',e'} = \frac{1}{2}$. If $e'\in R_{t-1}$, both $x$ and $y$ will have $0$ weight on the entry corresponding to $e'$, and the confidence radius related to $e'$ does not matter.
    
    Then we have
    \begin{align*}
        \text{InL} - \text{ExL} - \varepsilon_q \ge& \frac{1}{2} - 2 c_t - \left(\frac{1}{2} - \Delta_e + 2 c_t\right) - 2\varepsilon_q \\
        =& \Delta_e - 2\varepsilon_q - 4 c_t \\
        >& \Delta_e - \frac{2\Delta_e}{6} - \frac{4\Delta_e}{6}\\
        =& 0,
    \end{align*}
    where we use the assumption that $\Delta_e > 6\varepsilon_q, c_t < \frac{\Delta_{e}}{6}$.
    
    \paragraph{Case 2.} If arm $e\notin M_*^C,\Delta_e > 6\varepsilon_q, c_t < \frac{\Delta_{e}}{6}$, and $e\notin R_{t-1}$, we show that $e\in R_{t}$.
    \begin{align*}
        \text{ExL} + \varepsilon_q \ge& \max_{x\in\cP(\cM, A_{t-1}, R_{t-1}\cup\{e\})}\min_{y\in\cP(\cM, A_{t-1}, R_{t-1})}\frac{1}{\ell}x^T\underline{P}y \\
        \ge& \min_{y\in\cP(\cM, A_{t-1}, R_{t-1})}\frac{1}{\ell}\chi_{M_*^C}^T\underline{P}y \\
        \ge& \min_{y\in\cP(\cM, A_{t-1}, R_{t-1})}\frac{1}{\ell}\chi_{M_*^C}^TPy - \frac{1}{\ell}\ell 2c_t \\
        \ge& \frac{1}{2} - 2 c_t,
    \end{align*}
    and
    \begin{align*}
        \text{InU} \le & \max_{x\in\cP(\cM, A_{t-1}\cup\{e\}, R_{t-1})}\min_{y\in\cP(\cM, A_{t-1}, R_{t-1})}\frac{1}{\ell}x^T\bar P y \\
        \le& \max_{x\in\cP(\cM, A_{t-1}\cup\{e\}, R_{t-1})}\frac{1}{\ell}x^T\bar P \chi_{M_*^C} \\
        \le& \max_{x\in\cP(\cM, A_{t-1}\cup\{e\}, R_{t-1})}\frac{1}{\ell}x^T\bar P \chi_{M_*^C} + \frac{1}{\ell}\ell 2c_t \\
        \le& \frac{1}{2} - \Delta_e + 2 c_t.
    \end{align*}
    Then we have
    \begin{align*}
        \text{ExL} - \text{InL} - \varepsilon_q \ge& \frac{1}{2} - 2\cdot c_t - \left(\frac{1}{2} - \Delta_e + 2 c_t\right) - 2\varepsilon_q \\
        =& \Delta_e - 2\varepsilon_q - 4 c_t \\
        >& \Delta_e - \frac{2\Delta_e}{6} - \frac{4\Delta_e}{6}\\
        =& 0,
    \end{align*}
    where we use the assumption that $\Delta_e > 6\varepsilon_q, c_t < \frac{\Delta_{e}}{6}$.
    
    Now we bound the round $t_e$ such that an edge $e$ is added to $A_{t_e+1}$ or $R_{t_e+1}$. Note that previously, we prove that when $\cN$ at round $t$ in epoch $q$ such that $e\in U_t$ and $c_t < \frac{\Delta_{e}}{6}$ for $\Delta_e > 6\varepsilon_q$, the algorithm $\Carcond$ will classify arm $e$ into either $A_{t+1}$ or $R_{t+1}$. Note that when we select $t'_e = \frac{162}{(\Delta_e^C)^2}\ln\left(\frac{162K}{\delta(\Delta_e^C)^2}\right)$, we know that $t'_e$ is in epoch $q'_e$ such that $\varepsilon_{q'_e} \le \frac{\Delta^C_e}{6}$ since
    \[\frac{1}{\left(\frac{(\Delta^C_e)^2}{6}\right)^2} < \frac{162}{(\Delta_e^C)^2}\ln\left(\frac{162K}{\delta(\Delta_e^C)^2}\right).\]
    Recall that the confidence radius is defined as follow $c_t = \sqrt{\frac{\ln(4Kt^3/\delta)}{2t}}$, and we have the following
    \begin{align*}
        c_{t'_e} =& \sqrt{\frac{\ln(4K(t'_e)^3/\delta)}{2t'_e}} \\
        =& \sqrt{\frac{\ln(4K/\delta)}{2t'_e} + \frac{3\ln(t'_e)}{2t'_e}} \\
        \le& \sqrt{\frac{\ln(4K/\delta)}{2\frac{162}{(\Delta_e^C)^2}\ln\left(\frac{162K}{\delta(\Delta_e^C)^2}\right)}+ \frac{3\ln(162/(\Delta_e^C)^2) + 3\ln\ln\left(\frac{162K}{(\delta\Delta_e^C)^2}\right)}{2\frac{162}{(\Delta_e^C)^2}\ln\left(\frac{162K}{\delta(\Delta_e^C)^2}\right)}}\\
        <&\sqrt{\frac{(\Delta_e^C)^2}{2\times 162} + \frac{3(\Delta_e^C)^2}{2\times 162} + \frac{3(\Delta_e^C)^2}{2\times 162}} \\
        <&\sqrt{(\Delta_e^C)^2 \times \frac{3}{108}}\\
        =&\frac{\Delta^C_e}{6}.
    \end{align*}
    Also note that $c_t$ is monotonically decreasing when $t$ increases and $t\ge 3$, and $\varepsilon_q$ (as a function of $t$) is also monotonically decreasing when $t$ increases, so we know that $t_e \le t'_e$ when $t'_e \ge 3$.
    
    Now from the definition of our algorithm, we will sample edges $e_1\neq e_2$ if and only if they are connected to the same position and $t \le \min\{t_{e_1},t_{e_2}\}$. Combining the previous bound on $t_{e_1},t_{e_2}$, we know that when $\cN$ happens, the sample complexity of $\Carcond$ is bounded by
    \[O\left(\sum_{j=1}^{\ell}\sum_{e_1\neq e_2,e_1,e_2\in E_j}\frac{1}{(\Delta^C_{e_1,e_2})^2}\ln\left(\frac{K}{\delta(\Delta^C_{e_1,e_2})^2}\right)\right).\]
\end{proof}

Combining the Correctness lemma (Lemma \ref{lem:carcond-correct}), the Sample Complexity lemma (Lemma \ref{lem:carcond-sample}), and the fact that $\cN$ is a high probability event (Lemma \ref{lem:sampling-nice-high-prob}), we have the following theorem.

\thmcarcond*

\subsection{Details for the oracle implementation}\label{sec:oracle}

In this section, we introduce the implementation of the oracle used in $\Carcond$. Recall that we use the following oracle: The oracle can approximately solve the following optimization
\begin{equation*}
    \max_{x\in \cP(\cM,A_1,R_1)}\min_{y\in\cP(\cM,A_2,R_2)}\frac{1}{\ell}x^TQy,
\end{equation*}
where $\cP(\cM, A, R) = \{\sum_i\lambda_i\chi_{M_i}:M_i\in\cM,A\subset M_i, R\subset (M_i)^c,\sum_{i}\lambda_i = 1\}$ is the convex hull of the vector representations of the matchings, such that all the edge $A$ are included in the matching and all of $R$ are not included in the matching.

First, we give the full detailed algorithm for the implementation of the oracle. Algorithm \ref{alg:minimax-offline-detailed} is the main algorithm and Algorithm \ref{alg:app-pro-fw} is the approximation algorithm.

\begin{algorithm}
\caption{Condorcet Oracle (Detailed)}
\label{alg:minimax-offline-detailed}
\begin{algorithmic}[1]
    \STATE {\bfseries Input:} Bipartite graph $G$, weight matrix $W$ where $w_{i,j}$ denote an estimation of the probability that $i$ wins $j$, Accepted/Rejected Set for $x,y$: $A_x,R_x,A_y,R_y$
    \STATE {\bfseries Goal:} Find the approximate optimal solution of $\max_{x\in\cP( \cM, A_{x}, R_{x})}f_{A_y,R_y}(x)$
    \STATE Initialize $x^{(1)} = \chi_M$ for any possible $M$ such that $A_x\subset M$ and $R_x\subset M^c$
    \STATE Time hozizon $T = \lceil\frac{(4\ell K)^2}{\varepsilon^2}\rceil$, Step size $\eta = \frac{2\ell}{K\sqrt{M}}$, Accuracy $\varepsilon' = \frac{\varepsilon}{2K\sqrt{M}}$ for the approximate projection oracle.
    \FOR{$t=1,2,\dots,T$}
        \STATE Compute the subgradient $\nabla f_{A_y,R_y}(x)$ at the point $x^{(t)}$
        \STATE $y^{(t+1)} \leftarrow x^{(t)} + \eta \nabla f_{A_y,R_y}(x^{(t)})$
        \STATE $x^{(t+1)} \leftarrow \Pi_{\varepsilon'}(y^{(t+1)},\cP( \cM, A_{x}, R_{x}))$
    \ENDFOR
    \STATE {\bfseries return} $\left(f_{A_y,R_y}(x^{(t+1)}),x^{(t+1)}\right)$
\end{algorithmic}
\end{algorithm}

\begin{algorithm}
\caption{Approximate projection by Frank-Wolfe}
\label{alg:app-pro-fw}
\begin{algorithmic}[1]
    \STATE {\bfseries Input:} Point $x\in \R^K$, Bipartite graph $G$ with maximum matching $\ell$, Accepted set $A$ and Rejected set $R$, Accuracy Parameter $\varepsilon$
    \STATE {\bfseries Output:} Approximate projection $y$ such that $||y-\Pi(x,\cP(\cM, A, R)||_2 \le \varepsilon$
    \STATE $x^{(1)} \leftarrow \chi_M$, where $M$ is any maximum cardinal matching for graph $G$.
    \FOR{$t=1,2,\dots,\lceil 16\ell^2/\varepsilon^2\rceil$}
        \STATE $c\leftarrow x^{(t)} - x$
        \STATE Solve the minimum cost maximum matching for graph $G$ with cost vector $c$
        \STATE Denote the solution as $\chi_t$
        \STATE $x^{(t+1)}\leftarrow \left(1-\frac{2}{t+1}\right)x^{(t)} + \frac{2}{t+1}\chi_t$
    \ENDFOR
    \STATE {\bfseries return} $x^{(\lceil 8\ell^2/\varepsilon^2\rceil+1)}$
\end{algorithmic}
\end{algorithm}

Recall that the general idea for the implementation of our oracle is to apply the projected sub-gradient descent, and while in the projection step, we use the Frank-Wolfe algorithm to perform the approximate projection step. The proof is organized as follow: 1. We first prove that the function we optimize $\min_{y\in\cP(\cM,A_2,R_2)}\frac{1}{\ell}x^TQy$ is a concave function and has the properties that we need to use in the proof (Bounded (Lemma \ref{lem:function-bounded}) and Lipschitz (Lemma \ref{lem:function-lipschitz})). After the basic properties, we show the main lemma of the approximation projection (Lemma \ref{lem:approximation-projection-detailed}). Finally, we combine the projected sub-gradient descent with the approximation oracle (Lemma \ref{lem:minimax-oracle-main}). 

\begin{lemma}\label{lem:function-bounded}
    For any Accepted/Rejected sets $A,R$, the diameter of the set $\cP(\cM,A,R)$ is bounded by $2K$. Formally, we have
    \[\sup_{x,y\in\cP(\cM,A,R)}||x-y||_2 \le 2\ell.\]
\end{lemma}

\begin{proof}
    First note that, for any $x\in \cP(\cM,A,R)$, we have $||x||_1 = \ell$, because $x$ is a linear combination of matching with cardinal $\ell$. Then we have
    \begin{align*}
        \sup_{x,y\in\cP(\cM,A,R)}||x-y||_2 \le& \sup_{x,y\in\cP(\cM,A,R)}||x-y||_1 \\
        \le& \sup_{x,y\in\cP(\cM,A,R)}(||x||_1 + ||y||_1) \\
        \le& \ell + \ell \\
        =& 2\ell.
    \end{align*}
\end{proof}

\begin{lemma}\label{lem:function-lipschitz}
    Fixing the matrix $W$, the accepted/rejected sets $A_y,R_y$, the function
    \[f_{A_y,R_y}(x) = \min_{y\in\cP( \cM, A_{y}, R_{y})}\frac{1}{\ell}x^TW y\]
    is concave and $K$-Lipschitz.
\end{lemma}

\begin{proof}
    First, we know that $f_{A_y,R_y}(x)$ is concave, because
    \[f_{A_y,R_y}(x) = \min_{y\in \cP(\cM,A_y,R_y)} \frac{1}{\ell}x^TWy,\]
    is the minimum of linear functions, and thus is concave. Furthermore, we show that $f_{A_y,R_y}(x)$ is $K$-Lipschitz. For any $x_1,x_2$, let $y_2 = \argmin_{y\in\cP( \cM, A_{y}, R_{y})}\frac{1}{\ell}x_2^TW y$, and we have
    \begin{align*}
        f_{A_y,R_y}(x_1)-f_{A_y,R_y}(x_2) =& \min_{y\in\cP( \cM, A_{y}, R_{y})}\frac{1}{\ell}x_1^TW y-\min_{y\in\cP( \cM, A_{y}, R_{y})}\frac{1}{\ell}x_2^TW y \\
        \le& \frac{1}{\ell}x_1^TW y_2-\frac{1}{\ell}x_2^TW y_2 \\
        =& (x_1-x_2)^T \frac{1}{\ell}W y_2 \\
        \le& ||x_1-x_2||_2\cdot ||\frac{1}{\ell}W y_2||_2 \\
        \le& K||x_1-x_2||_2,
    \end{align*}
    where the last inequality comes from the fact that each entry in $W$ belongs to $[0,1]$ and the 1-norm of $y_2$ is $||y_2||_1 = \ell$. Similarly, we can also prove that
    \[f_{A_y,R_y}(x_2)-f_{A_y,R_y}(x_2) \le K||x_1-x_2||_2,\]
    and we can conclude that $f_{A_y,R_y}(x)$ is $K$-Lipschitz.
\end{proof}

Then, we come to the proof of the approximation oracle. First we recall the procedure of the Frank-Wolfe Algorithm and recall the performance guarantee. Then we recall the projection lemma that we use in the analysis. We refer to Section \ref{sec:appendix-background} for more background on Frank-Wolfe Algorithm and other basic properties of convex optimization.

For a convex function $f$ defined on a convex set $\cX$, given a fixed sequence $\{\gamma_t\}_{t\ge 1}$, the Frank-Wolfe Algorithm iterate as the following for $t\ge 1$:
\begin{align*}
    y^{(t)} \in& \arg\min_{y\in \cX}\nabla f(x^{(t)})^Ty \\
    x^{(t+1)} =& (1-\gamma_t)x^{(t)} + \gamma_t y^{(t)}
\end{align*}

The following is the performance guarantee of the Frank-Wolfe algorithm.

\propfw*

Also recall that we have the following property for projecting to a convex set.

\propproj*

Now we give the lemma of the approximation projection. The lemma is nearly a direct application of the proposition of the Frank-Wolfe performance guarantee and the projection proposition, but we need to carefully choose the parameters.

\begin{lemma}[Approximate Projection]\label{lem:approximation-projection-detailed}
    Let $\Pi(x,\cP(\cM,A,R))$ denote the projection of $x$ onto the distribution polytope $\cP(\cM,A,R)$. Algorithm \ref{alg:app-pro-fw} will return a solution $x_r$ such that $||x_r - \Pi(x,\cP(\cM,A,R))||_2 \le \varepsilon$. Moreover, $x_r$ can be represented by $\sum_e \lambda_e\chi_{M_e}$ such that $M\in\cM,A\subseteq M_e, R\subseteq M_e^*$ and $\boldsymbol{\lambda}$ is sparse.
\end{lemma}

\begin{proof}
    Denote $x_r = x^{(\lceil 16\ell^2/\varepsilon^2\rceil+1)}$. First we know that $x_r$ is a linear combination of the vertices, and it is easy to see that the coefficient vector $\boldsymbol{\lambda}$ can have at most $\lceil 16\ell^2/\varepsilon^2\rceil$ non-zero entries. Thus, we know that $x_r\in \cP(\cM,A,R)$. From the property of Frank-Wolfe algorithm (Proposition \ref{prop:fw}), we know that
    \[\frac{1}{2}||x-x_r||_2^2\le \frac{1}{2}||x-\Pi(x,\cP(\cM,A,R))||_2^2 + \frac{1}{2}\varepsilon^2,\]
    since $D = \sup_{x,y\in\cP(\cM,A,R)}||x-y||_2 \le 2\ell$ and the function $f(y) = \frac{1}{2}||x-y||_2^2$ is $1$-smooth. Then, from the property of projection (Proposition \ref{prop:proj}), we know that
    \[||x_r - x||_2^2 \ge ||x-\Pi(x,\cP(\cM,A,R))||_2^2 + ||x_r-\Pi(x,\cP(\cM,A,R))||_2^2.\]
    Then we know that
    \[||x_r-\Pi(x,\cP(\cM,A,R))||_2^2 \le \varepsilon^2,\]
    and complete the proof of this lemma.
\end{proof}

By the help of the previous lemmas, we have the following main lemma for our minimax oracle. The main lemma follows the proof strategy of the projected sub-gradient descent, but we need to substitute the original accurate projection oracle to our approximate projection oracle.

\begin{lemma}[Minimax Oracle]\label{lem:minimax-oracle-main}
    Using the Minimax Oracle (Algorithm \ref{alg:minimax-offline-detailed}) with the approximate projection oracle (Algorithm \ref{alg:app-pro-fw}), the output $(f_{A_y,R_y}(x_r),x_r)$ satiesfies
    \[f_{A_y,R_y}(x_r) \ge \max_{x\in\cP(\cM,A_x,R_x)}f_{A_y,R_y}(x) - \varepsilon.\]
\end{lemma}

\begin{proof}
    Let $T = \frac{(2\ell K)^2}{\varepsilon^2}$ denote the total steps, $\eta = $ denote the step size, and $\varepsilon' = $ denote the accuracy of the approximate projection oracle. We show that
    \[f_{A_y,R_y}(x^*) - \max_{t\le T} f_{A_y,R_y}(x^{(t)}) \le \varepsilon.\]
    We have
    \begin{align*}
        f(x^*) - f(x^{(t)}) \le& \nabla f(x^{(t)})^T(x^* - x^{(t)}) \\
        =& \frac{1}{\eta}(y^{(t+1)}-x^{(t)})^T(x^*-x^{(t)})\\
        =& \frac{1}{2\eta}\left(||x^{(t)}-x^*||_2^2 + ||x^{(t)}-y^{(t+1)}||_2^2 - ||y^{(t+1)}-x^*||_2^2\right)\\
        =& \frac{1}{2\eta}\left(||x^{(t)}-x^*||_2^2 - ||y^{(t+1)}-x^*||_2^2\right) + \frac{\eta}{2}||f(x^{(t)})||_2^2.
    \end{align*}
    Note that from Proposition \ref{prop:proj}, we have
    \[||y^{(t+1)}-x^*||_2^2 \ge ||\Pi(y^{(t+1)},\cP(\cM,A_x,R_x))-x^*||_2^2.\]
    Furthermore, since
    \[||x^{(t+1)} - \Pi((y^{(t+1)},\cP(\cM,A_x,R_x))||_2 \le \varepsilon',\]
    we have
    \begin{align*}
        &\big| ||x^{(t+1)}-x^*||_2^2 - ||\Pi(y^{(t+1)},\cP(\cM,A_x,R_x))-x^*||_2^2\big|\\
        =& \big| (x^{(t+1)}-\Pi(y^{(t+1)},\cP(\cM,A_x,R_x)))^T(x^{(t+1)} + \Pi(y^{(t+1)},\cP(\cM,A_x,R_x)) - 2x^*)\big|\\
        \le& ||x^{(t+1)}-\Pi(y^{(t+1)},\cP(\cM,A_x,R_x))||_2\cdot ||x^{(t+1)} + \Pi(y^{(t+1)},\cP(\cM,A_x,R_x))-2x^*||_2\\
        \le& ||x^{(t+1)}-\Pi(y^{(t+1)},\cP(\cM,A_x,R_x))||_2\cdot \left(||x^{(t+1)}-x^*||_2 + ||\Pi(y^{(t+1)},\cP(\cM,A_x,R_x))-x^*||_2\right)\\
        \le& 4\ell\varepsilon',
    \end{align*}
    where in the last step, we use the fact that
    \[\sup_{x,y\in\cP(\cM,A_x,R_x)}||x-y||_2 \le 2\ell.\]
    Sum up all $t\le T$, apply the fact that $||\nabla f_{A_y,R_y}(x)||_2\le K$ (because the function $f_{A_y,R_y}(x)$ is $K$-Lipschitz), we have
    \begin{align*}
        \sum_{i=1}^T\left(f(x^*) - f(x^{(t)})\right) \le& \sum_{i=1}^T\left(\frac{1}{2\eta}\left(||x^{(t)}-x^*||_2^2 - ||y^{(t+1)}-x^*||_2^2\right) + \frac{\eta}{2}||f(x^{(t)})||_2^2\right)\\
        \le& \sum_{i=1}^T\left(\frac{1}{2\eta}\left(||x^{(t)}-x^*||_2^2 - ||x^{(t+1)}-x^*||_2^2 + 4\ell\varepsilon' \right) + \frac{\eta}{2}||f(x^{(t)})||_2^2\right)\\
        \le& \frac{4\ell^2}{2\eta} + \frac{M}{2\eta}4\ell\varepsilon'  + \frac{T\eta}{2}K^2,
    \end{align*}
    and we can get
    \begin{align*}
        f_{A_y,R_y}(x^*) - \max_{t\le T} f_{A_y,R_y}(x^{(t)}) \le& \frac{1}{M}\sum_{i=1}^T\left(f(x^*) - f(x^{(t)})\right)\\
        \le& \frac{4\ell^2}{2T\eta} + \frac{1}{2\eta}4\ell\varepsilon'  + \frac{\eta}{2}K^2.
    \end{align*}
    Plug in $T = \lceil\frac{(4\ell K)^2}{\varepsilon^2}\rceil,\eta = \frac{2\ell}{K\sqrt{M}},\varepsilon' = \frac{\varepsilon}{2K\sqrt{M}}$, we can get
    \begin{align*}
        f_{A_y,R_y}(x^*) - \max_{t\le T} f_{A_y,R_y}(x^{(t)}) \le& \frac{4\ell^2}{2T\eta} + \frac{1}{2\eta}4\ell\varepsilon'  + \frac{\eta}{2}K^2\\
        =& \frac{4\ell^2}{2 T \frac{2\ell}{K\sqrt{M}}} + \frac{2\ell\varepsilon'}{\frac{2\ell}{K\sqrt{M}}} + \frac{\frac{2\ell}{K\sqrt{M}}}{2}K^2 \\
        =& \frac{2\ell K}{\sqrt{M}} + K\sqrt{M}\varepsilon' \\
        \le& 2\ell K\frac{\varepsilon}{4\ell K} + \frac{\varepsilon}{2} \\
        =&\varepsilon.
    \end{align*}
\end{proof}

\subsection{Details of the verification algorithm}
\label{sec:appendix_cond_verify}

Recall that we introduce the following definitions.

For any $e \notin M^C_*$, we define the verification gap $\tilde{\Delta}^C_{e}$ as 
\[
\tilde{\Delta}^C_{e}= \min_{M \in \cM \setminus \{M^C_*\}:e \in M} \left\{ \frac{\ell}{d_{M_*^C, M}} \cdot \left(\frac{1}{2} - \frac{1}{\ell}\chi_M^T P\chi_{M_*^C} \right) \right\},
\]
where $d_{M_x, M_y}$ denotes the number of positions with different edges between $M_x$ and $M_y$, \emph{i.e.}, $d_{M_x, M_y}:= \sum_{j=1}^{\ell} \mathbb{I}\{e(M_x,j) \neq e(M_y,j)\}$.

For ease of notation, we define the following quantity
$$
H^C_{\textup{ver}}:=\sum_{e \notin M^C_*} \frac{1}{(\tilde{\Delta}^C_{e})^2}.
$$

Next, we present two lemmas for $\Carverify$ on the sample complexity and correctness with high probability.

\begin{restatable}[$\Carverify$]{lemma}{lemmaverify}\label{lemma:verify}
    Assume the existence of Condorcet winner. Then, with probability at least $1-\delta_0-\delta$, the $\Carverify$ algorithm (Algorithm \ref{alg:condorcet-verify}) will return the  Condorcet winner with sample complexity 
\begin{align*}
O\left(
    \sum_{j=1}^{\ell}\sum_{\begin{subarray}{c} e\neq e' \\ e,e'\in E_j\end{subarray}}\frac{1}{(\Delta^C_{e,e'})^2}\ln\left(\frac{K}{(\Delta^C_{e,e'})^2}\right)
    + H^C_{\textup{ver}} \ln\left(\frac{H^C_{\textup{ver}}}{\delta}  \right)
    \right).    
\end{align*}
\end{restatable}

\begin{proof}
First, we define event $\mathcal{E}:=\{ \hat{M}=M^C_*\}$. From Theorem \ref{thm:carcond}, we have $\Pr[\mathcal{E}]\geq 1-\delta_0$.
We also define the event $\mathcal{F}_t:=\{  |\hat{p}_{e_i,e_j}-p_{e_i,e_j}|< c_{e_i,e_j}(t), \forall e_i \neq e_j, s(e_i)=s(e_j) \}$ for any timestep $t$.
Since $c_t(e_i,e_j) = \sqrt{\frac{\ln(4Kt^3/\delta)}{2T_t(e_i,e_j)}}$, from the Chernoff-Hoeffding bound, we can obtain that for any $t$, for any $e_i, e_j, \ s.t. \  e_i \neq e_j, s(e_i)=s(e_j)$,
\begin{align*}
	\Pr [|\hat{p}_{e_i,e_j}-p_{e_i,e_j}|\geq c_{e_i,e_j}(t) ] = &  \sum \limits_{s=1}^{t}  \Pr \left[ |\hat{p}_{e_i,e_j}-p_{e_i,e_j}|\geq  \sqrt{\frac{\log (\frac{4K t^3}{\delta})}{2 s}},  T_t(e_i,e_j)=s  \right]   
	\\
	 \leq  & \sum \limits_{s=1}^{t}  \frac{\delta}{2 K t^3}  
	 \\
	 \leq  &  \frac{\delta}{2 K t^2}.  
\end{align*}
By a union bound over $e_i,e_j$, we have 
$\Pr [ \overline{\mathcal{F}_t} ] \leq  \frac{\delta}{2 t^2}$.

Define event $\mathcal{F} := \bigcap \limits_{t=1}^{\infty} \mathcal{F}_t$.
Then, we have
$\Pr [\mathcal{F} ] \geq    1- \sum \limits_{t=1}^{\infty}  \Pr [\overline{\mathcal{F}_t} ]  
\geq   1-   \sum \limits_{t=1}^{\infty} \frac{\delta}{2 t^2}  
\geq    1- \delta  
$.

Below we prove that for any $e_i \notin M^C_*$, let $e_j$ be the edge in $M^C_*$ at the same position as $e_i$, \emph{i.e.}, $e_j \in M^C_*, s(e_i)=s(e_j)$, and then conditioning on $\mathcal{E} \cap \mathcal{F}$, when $c_t(e_i,e_j)<\frac{1}{2}\tilde{\Delta}^C_{e_i}$, the duel $(e_i,e_j)$ will not be pulled.

Suppose that, $\mathcal{E} \cap \mathcal{F}$ occur, and at some timestep $t$, $c_t(e_i,e_j)<\frac{1}{2}\tilde{\Delta}^C_{e_i}$ and  $\Carverify$  pulls the duel $(e_i,e_j)$, \emph{i.e.}, $(e_t, f_t)=(e_i,e_j)$. Then, from the occurences of $\mathcal{E} \cap \mathcal{F}$ and the definition of $\tilde{\Delta}^C_{e_i}$, we have 
\begin{align*}
c_t(e_i,e_j) < & \frac{1}{2} \cdot \min_{M \in \cM \setminus \{M^C_*\}:e_i \in M} \left\{ \frac{\ell}{d_{M_*^C, M}} \cdot \left(\frac{1}{2} - \frac{1}{\ell}\chi_M^T P\chi_{M_*^C} \right) \right\}
\\
\leq &  \frac{\ell}{2 d_{M_*^C, M_t}} \cdot \left(\frac{1}{2} - \frac{1}{\ell}\chi_{M_t}^T P\chi_{M_*^C} \right).    
\end{align*}
According to the selection of $(e_t, f_t)$ in $\Carverify$, we have that for any $e,e' \ s.t. \  e \in M_t \setminus M_*^C, e' \in M_*^C \setminus M_t, s(e)=s(e')$,
\begin{align*}
c_t(e,e') \leq & c_t(e_i,e_j) 
\\
<  &  \frac{\ell}{2 d_{M_*^C, M_t}} \cdot \left(\frac{1}{2} - \frac{1}{\ell}\chi_{M_t}^T P\chi_{M_*^C} \right).    
\end{align*}

Thus, we have
\begin{align*}
f(M_t, M_*^C, \bar P_t)
< & f(M_t, M_*^C,  P_t) + \frac{2}{\ell} \sum_{\begin{subarray}{c}  e \in M_t \setminus M_*^C,  e' \in M_*^C \setminus M_t \\ s(e)=s(e') \end{subarray}} c_t(e', e)
\\
<  &   \frac{1}{\ell}\chi_{M_t}^T P\chi_{M_*^C}  + \frac{2}{\ell} \cdot d_{M_*^C, M_t} \cdot \frac{\ell}{2 d_{M_*^C, M_t}} \cdot \left(\frac{1}{2} - \frac{1}{\ell}\chi_{M_t}^T P\chi_{M_*^C} \right)
\\
= & \frac{1}{2},
\end{align*}
which contradicts the return condition of $\Carcond$.

Thus, conditioning on $\mathcal{E} \cap \mathcal{F}$, when $c_t(e_i,e_j)<\frac{1}{2}\tilde{\Delta}^C_{e_i}$, the duel $(e_i,e_j)$ will not be pulled. Let $T_{\textup{cond}}$ and $T_{\textup{ver}}$ denote the number of samples incurred by the sub-procedure $\Carcond(\delta_0)$ and  the verification part (from Line 4 to end), respectively.
Then, using the similar analysis as the proof of Theorem \ref{thm:borda_ub}, we have that for any $e,e' \ s.t. \  e \notin M_*^C , e' \in M_*^C , s(e)=s(e')$
\begin{align*}
    T(e,e') \leq \frac{1}{(\tilde{\Delta}^C_{e})^2} \ln \left(\frac{4K T^3}{\delta} \right)+1
\end{align*}
Note that fixing $e \notin M_*^C$, $e'$ is the edge in $M^C_*$ at the same position as $e$, \emph{i.e.}, $e' \in M_*^C , s(e)=s(e')$.
Thus, taking summation over $e \notin M_*^C$, we have
\begin{align*}
    T_{\textup{ver}} \leq H^C_{\textup{ver}} \ln \left(\frac{4K T^3}{\delta} \right)+1
\end{align*}
Thus, we can obtain $T_{\textup{ver}}=O(H^C_{\textup{ver}} \ln (\frac{H^C_{\textup{ver}}}{\delta}) )$.
Then, from Theorem \ref{thm:carcond}, we have that conditioning on $\mathcal{E} \cap \mathcal{F}$,
\begin{align*}
T= & T_{\textup{cond}}+T_{\textup{ver}}
\\
= &  O\left(\sum_{j=1}^{\ell}\sum_{\begin{subarray}{c} e\neq e' \\ e,e'\in E_j\end{subarray}}\frac{1}{(\Delta^C_{e,e'})^2}\ln\left(\frac{K}{\delta_0 (\Delta^C_{e,e'})^2}\right)\right)  + O \left(H^C_{\textup{ver}} \ln \left(\frac{H^C_{\textup{ver}}}{\delta} \right)  \right)
\\
= & O\left(
    \sum_{j=1}^{\ell}\sum_{\begin{subarray}{c} e\neq e' \\ e,e'\in E_j\end{subarray}}\frac{1}{(\Delta^C_{e,e'})^2}\ln\left(\frac{K}{(\Delta^C_{e,e'})^2}\right)
    + H^C_{\textup{ver}} \ln\left(\frac{H^C_{\textup{ver}}}{\delta}  \right)
    \right),
\end{align*}
which completes the proof of Lemma \ref{lemma:verify}.
\end{proof}

\begin{restatable}[$\Carverify$-correctness]{lemma}{lemmaverifycorrect}\label{lemma:verifycorrect}
    Assume the existence of Condorcet winner. Then, with probability at least $1-\delta$,  the $\Carverify$ algorithm (Algorithm \ref{alg:condorcet-verify}) will return the  Condorcet winner or an error.
\end{restatable}

\begin{proof}
Recall that $\Pr[\mathcal{F}] \geq 1-\delta$.

Then, conditioning on $\mathcal{F}$, if $\Carverify$ terminates  with an error, Lemma \ref{lemma:verifycorrect} holds. If $\Carverify$ terminates  with an answer $\Out=M_t$, we have $f(M, \hat{M},  P_t) < f(M, \hat{M}, \bar P_t) \leq \max_{M \in \cM \setminus \{\hat{M}\} }f(M, \hat{M}, \bar P_t) \leq \frac{1}{2}$ for any $M \in \cM \setminus \{\hat{M}\}$, and thus the answer $\Out=M_t=M^C_*$.

Note that conditioning on $\mathcal{F}$, $\Carverify$ must terminate. This is because
if $\mathcal{F} \cap \mathcal{E}$ occur, according to Lemma \ref{lemma:verify}, $\Carverify$ will terminate and return the Condorcet winner with a bounded samples. Otherwise, if $\mathcal{F} \cap \bar{\mathcal{E}}$ occur, we have that $M^C_* \in \cM \setminus \{ \hat{M} \}$ and $f(M^C_*, \hat{M},  P_t)> \frac{1}{2}$. Then, the condition of returning an answer cannot be satisfied and the condition of returning an error will be satisfied with limit timesteps because the confidence radius shrinks as the timestep increases.

Therefore, we complete the proof of Lemma \ref{lemma:verifycorrect}.
\end{proof}

Now, we present the expected sample complexity for the $\Carparallel$ algorithm.

\thmparallel*

\begin{proof}
Since $\Carparallel$ directly applies the ``parallel simulation'' technique \cite{parallel_ChenLJ2015,nearly_ChenLJ2017} on $\Carverify$ to boost the confidence,  
Theorem \ref{thm:parallel}  follows from Lemma \ref{lemma:verify}, \ref{lemma:verifycorrect} and Lemma 4.8 (result for parallel simulation) in \cite{nearly_ChenLJ2017}. 
\end{proof}

\subsection{Lower Bound}
\label{sec:appendix_cond_lb}
To formally state our result for lower bound, we first introduce the following notions. For any $\delta \in (0,1)$, we call an algorithm $\mathbb{A}$ a $\delta$-correct algorithm if, for any problem instance of CPE-DB with Condorcet winner, algorithm $\mathbb{A}$ identifies the Condorcet winner with probability at least $1-\delta$.
In addition, for any $M \in \cM \setminus \{M^C_*\}$, we use $\mathcal{O}(M)$ to denote the set of matchings that can beat $M$, \emph{i.e.}, $\mathcal{O}(M)=\{M_x \in \cM \setminus \{M\}: f(M_x, M, P) \}\geq \frac{1}{2}$. According to the definition of Condorcet winner, $ M^C_* \in \mathcal{O}(M)$ for any $M \in \cM \setminus \{M^C_*\}$.

In the following, we present a lower bound for the problem of combinatorial pure exploration for identifying the Condorcet winner in a special case.

\begin{restatable}[Condorcet lower bound]{theorem}{thmcondlb} \label{thm:cond_lb}
	Consider the problem of combinatorial pure exploration for identifying the Condorcet winner. Suppose that, for any $M \in \cM \setminus \{M^C_*\}$, for any $M_x \in \mathcal{O}(M)$, $f(M,M^C_*,P) \leq f(M,M_x,P)$ and $M^C_* \setminus M \subseteq M_x \setminus M$. For some constant $0<\gamma<\frac{1}{2(2+\ell)}$, for any $e_i,e_j \in E, s(e_i)=s(e_j)$,  $\frac{1}{2}-\gamma \leq p_{e_i,e_j} \leq \frac{1}{2}+\gamma$. Then, for any $\delta \in (0,0.1)$, any $\delta$-correct algorithm has sample complexity 
	$$
	\Omega \Bigg(\sum_{e \notin M^C_*}  \frac{1}{\ell^2 \cdot ({\Delta}^C_{e})^2} \ln \Big(\frac{1}{\delta} \Big) \Bigg).
	$$
\end{restatable}

\begin{proof}

Fix an instance $\mathcal{I}$ of the Condorcet CPE-DB problem under the supposition and a $\delta$-correct algorithm $\mathbb{A}$. 
In instance $\mathcal{I}$, $M_*^C$ is the Condorcet winner and $M$ is a suboptimal matching.
Let $T_{e_i, e_j}$ be the expected number of samples drawn from the duel $(e_i, e_j)$ when $\mathbb{A}$ runs on instance $\mathcal{I}$.

We consider the following alternative instance $\mathcal{I'}$. For the duel $(e_i, e_j)$ such that $e_i \in M \setminus M^C_*, e_j \in M^C_* \setminus M, s(e_i)=s(e_j)$, we change the Bernoulli distribution of duel $(e_i, e_j)$  as  follows:
$$ 
p'_{e_i, e_j}=p_{e_i, e_j}+ \ell \cdot \left( \frac{1}{2}-f(M,M^C_*,P) + \lambda \right)
$$
Then, $f'(M,M^C_*,P)>\frac{1}{2}$. For any $M_x \in \mathcal{O}(M)$, since $f(M,M^C_*,P)<f(M,M_x,P)$ and $e_j \in M_x \setminus M$, we have $f'(M,M_x,P)>f(M,M_x,P)+ ( \frac{1}{2}-f(M,M^C_*,P))\geq \frac{1}{2}$.
Thus, we can see that in instance $\mathcal{I'}$, $M$ is the Condorcet winner instead.

Using Lemma 1 in \cite{kaufmann2016complexity}, we can obtain 
$$  T_{e_i, e_j} \cdot d(p_{e_i, e_j}, p'_{e_i, e_j}) \geq d(1-\delta, \delta).$$

For $\delta \in (0,0.1)$, we have $d(1-\delta,  \delta) \geq 0.4 \ln (\frac{1}{\delta} )$.
From the supposition, for some constant $0<\gamma<\frac{1}{2(2+\ell)}$, for any $e_i,e_j \in E, s(e_i)=s(e_j)$,  $\frac{1}{2}-\gamma \leq p_{e_i,e_j} \leq \frac{1}{2}+\gamma$.
Then, for any $M_1,M_2 \in \cM \ s.t. \  M_1 \neq M_2$,  $\frac{1}{2}-\gamma \leq f(M_1,M_2,P) \leq \frac{1}{2}+\gamma$. Thus, for the changed duel $(e_i,e_j)$, $ \gamma \leq p'_{e_i,e_j} \leq 1-\gamma$ and $d(p_{e_i,e_j}, p'_{e_i,e_j}) \leq \frac{ ( p_{e_i,e_j}-p'_{e_i,e_j} )^2 }{p'_{e_i,e_j} ( 1-p'_{e_i,e_j} )} \leq \frac{1}{\gamma ( 1-\gamma )}  ( p_{e_i,e_j}-p'_{e_i,e_j} )^2 $.

Therefore, 
$$
\frac{1}{\gamma ( 1-\gamma )} \cdot T_{e_i, e_j} \cdot  (p_{e_i, e_j} - p'_{e_i, e_j} )^2  \geq 0.4 \ln \Big(\frac{1}{\delta} \Big)
$$
$$
\frac{1}{\gamma ( 1-\gamma )} \cdot T_{e_i, e_j} \cdot  \ell^2 \cdot \left( \frac{1}{2}-f(M,M^C_*,P) + \lambda \right)^2  \geq 0.4 \ln \Big(\frac{1}{\delta} \Big)
$$
$$
T_{e_i, e_j} \geq  \frac{ 0.4 \gamma ( 1-\gamma )}{\ell^2 \cdot \left( \frac{1}{2}-f(M,M^C_*,P) + \lambda \right)^2} \ln \Big(\frac{1}{\delta} \Big)
$$

We can perform the similar distribution changes on any duel $(e_i, e_j)$ such that $e_i \in M \setminus M^C_*, e_j \in M^C_* \setminus M, s(e_i)=s(e_j)$ and any $M \in \cM \setminus \{M^C_*\}$. In addition, the inequality holds for any $\lambda>0$. Therefore, from the above analysis and the definition of $\Delta^C_{e_i}$ (Definition \ref{def:gap-carcond}), we can obtain that for any $e_i, e_j \in E$ such that $e_j \in M^C_*, e_i \notin M^C_*, s(e_i)=s(e_j)$,
\begin{align*}
T_{e_i, e_j}  \geq & \max_{M \in \cM \setminus \{M^C_*\}: e_i \in M} \left \{  \frac{0.4 \gamma ( 1-\gamma )}{\ell^2 \cdot \left( \frac{1}{2}-f(M,M^C_*,P)\right)^2} \ln \Big(\frac{1}{\delta} \Big) \right \} 
\\
\geq &   \frac{0.4 \gamma ( 1-\gamma )}{\ell^2 \cdot  (\Delta^C_{e_i})^2} \ln \Big(\frac{1}{\delta} \Big) 
\end{align*}

Thus, we can see that for any edge $e \notin M^C_*$, the number of samples for the duel between $e$ and  the edge in $M^C_*$  at the same position as $e$, which we denote by $T_{e}$, satisfies 
$$
T_{e} \geq    \frac{0.4 \gamma ( 1-\gamma )}{\ell^2 \cdot  (\Delta^C_{e})^2} \ln (\frac{1}{\delta} ).
$$

Summing over $e \notin M^C_*$, we have
\begin{align*}
T \geq & \sum_{e \notin M^C_*}  \frac{0.4 \gamma ( 1-\gamma )}{\ell^2 \cdot ({\Delta}^C_{e})^2} \ln \Big(\frac{1}{\delta} \Big)
\\
= & \Omega \Bigg(\sum_{e \notin M^C_*}  \frac{1}{\ell^2 \cdot ({\Delta}^C_{e})^2} \ln \Big(\frac{1}{\delta} \Big) \Bigg),
\end{align*}
which completes the proof of Theorem \ref{thm:cond_lb}.
\end{proof}

Note that in the sample complexity upper bound of $\Carparallel$ (Theorem \ref{thm:parallel}), for any $e \notin M^C_*$, the verification gap $\tilde{\Delta}^C_{e} \geq \bar{\Delta}^C_{e}$, and thus the verification hardness satisfies
$$
H^C_{\textup{ver}} 
\leq \sum_{e \notin M^C_*}  \frac{1}{ ({\Delta}^C_{e})^2}
$$
Thus, given confidence $\delta<0.01$,  the term $H^C_{\textup{ver}} \ln \left(\frac{H^C_{\textup{ver}}}{\delta} \right)$ in the sample complexity upper bound of Algorithm \ref{alg:condorcet-parallel} matches the lower bound within a factor of $\ell^2$.
}

\end{document}